\documentclass[11pt, letterpaper]{article}
\usepackage[utf8]{inputenc}

\usepackage[english]{babel}
\usepackage{titling}
\usepackage{a4wide}

\usepackage{cite}
\usepackage{amsmath}
\usepackage{amssymb}
\usepackage{amsthm}
\usepackage{amsfonts}       % blackboard math symbols

\usepackage{graphicx}

\usepackage{tikz}
\usetikzlibrary{arrows.meta}
\usetikzlibrary{calc}

\usetikzlibrary{arrows}
\usepackage{mathrsfs}

\usepackage{pgfplots}
\pgfplotsset{compat=1.15}

\usepackage[]{color}
\usepackage[cmyk]{xcolor}
\usepackage{authblk}
\usepackage{soul}
\usepackage{natbib}

\usepackage[utf8]{inputenc} % allow utf-8 input
\usepackage[T1]{fontenc}    % use 8-bit T1 fonts
\usepackage[urlcolor=blue, colorlinks=true,linkcolor=blue,citecolor=blue] {hyperref}     % hyperlinks
\usepackage{url}            % simple URL typesetting
\usepackage{booktabs}       % professional-quality tables

\usepackage{nicefrac}       % compact symbols for 1/2, etc.
\usepackage{microtype}      % microtypography
\usepackage{csquotes}

\usepackage{thmtools}
\usepackage{thm-restate}
\usepackage{graphics}
\usepackage{caption}
\usepackage{subcaption}
\usepackage{pgf}
\usepackage{algorithm}
\usepackage{algorithmic}

\usepackage{refcount}
\usepackage{fmtcount}
\usepackage{mleftright}

\usepackage{multirow}
\usepackage{tabularx}
\usepackage{multirow}
\usepackage{bbding}

\makeatletter
\newtheorem{theorem}{Theorem}
\newtheorem{proposition}{Proposition}
\newtheorem{observation}{Observation}
\newtheorem{definition}{Definition}
\newtheorem{lemma}{Lemma}

\newcommand{\opt}{\textsc{Opt}}
\newcommand{\alg}{\textsc{Alg}}
\newcommand{\pred}{\textsc{Trust}}
\newcommand{\I}{\mathcal{I}}
\newcommand{\trust}{\textsc{Trust}}
\newcommand{\greedy}{\textsc{Greedy}}
\newcommand{\SmoothMerge}{\textsc{SmoothMerge}}

\newcommand{\Iadversary}{\mathcal{I}_{ADV}}
\newcommand{\virtualalgorithm}{\textsc{Virtual Algorithm}}

\newcommand{\predGen}{*}
\newcommand{\predO}{\hat{\mathbf{O}}}

\newcommand{\predI}{\hat{\mathbf{I}}}

\newcommand{\firstPhase}{\textit{trust phase}}
\newcommand{\secondPhase}{\textit{independent phase}}

\definecolor{MyBlue}{cmyk}{1, 0.5, 0, 0}       % Dark Blue
\definecolor{MyGreen}{cmyk}{1, 0, 1, 0}        % Dark Green
\definecolor{MyRed}{cmyk}{0, 1, 1, 0.3}         % Deep Red (slightly muted to work in grayscale)

\newcommand{\drawInterval}[5][black]{%
    \draw[thick, #1] (#4,#3) -- (#5,#3); % Main interval line
    \draw[#1] (#4, #3 - 0.1) -- (#4, #3 + 0.1); % Left endpoint
    \draw[#1] (#5, #3 - 0.1) -- (#5, #3 + 0.1); % Right endpoint
}

\title{A Switching Framework for Online Interval Scheduling with Predictions}

\author[1]{Antonios Antoniadis\thanks{ \href{mailto:a.antoniadis@utwente.nl}{a.antoniadis@utwente.nl}}}
\author[2]{Ali Shahheidar\thanks{ \href{mailto:alishahheidar98@gmail.com}{alishahheidar98@gmail.com}}}
\author[3]{Golnoosh Shahkarami\thanks{ \href{mailto:gshahkar@mpi-inf.mpg.de}{gshahkar@mpi-inf.mpg.de}}}
\author[2]{Abolfazl Soltani\thanks{ \href{mailto:absoltani02@gmail.com}{absoltani02@gmail.com}}}

\affil[1]{University of Twente}
\affil[2]{Sharif University of Technology}
\affil[3]{Max Planck Institut für Informatik, Universität des Saarlandes}

\date{}
\begin{document}

\maketitle

\begin{abstract}
We study online interval scheduling in the irrevocable setting, where each interval must be immediately accepted or rejected upon arrival. The objective is to maximize the total length of accepted intervals while ensuring that no two accepted intervals overlap. We consider this problem in a learning-augmented setting, where the algorithm has access to (machine-learned) predictions. The goal is to design algorithms that leverage these predictions to improve performance while maintaining robust guarantees in the presence of prediction errors.

Our main contribution is the \textsc{SemiTrust-and-Switch} framework, which provides a unified approach for combining prediction-based and classical interval scheduling algorithms. This framework applies to both deterministic and randomized algorithms and captures the trade-off between consistency (performance under accurate predictions) and robustness (performance under adversarial inputs). Moreover, we provide lower bounds, proving the tightness of this framework in particular settings.

We further design a randomized algorithm that smoothly interpolates between prediction-based and robust algorithms. This algorithm achieves both robustness and smoothness--its performance degrades gracefully with the quality of the prediction.

\end{abstract}

\section{Introduction}
The online interval scheduling problem is a well-studied model in algorithm design that captures the challenge of selecting a subset of non-overlapping time intervals with the goal of optimizing a given objective. Intervals arrive sequentially over time, and the algorithm must irrevocably decide, upon each arrival, whether to accept or reject the interval without knowledge of future \textcolor{blue}{arrivals}. Each interval typically represents a task, job, or request, and one natural objective is to maximize the total length of the accepted intervals. This model arises in a wide range of practical applications, including computing resource management, manufacturing systems, and real-time service scheduling. Additional domains include vehicle routing and satellite communication scheduling~\citep{doi:10.1137/S0097539797321237}. Two surveys by~\cite{KolenLenstraPapadimitriouSpieksma2007, KOVALYOV2007331} provide comprehensive overviews of interval scheduling and its many application areas.

While classical online algorithms for interval scheduling, such as greedy approaches, offer strong worst-case guarantees, their performance can be overly pessimistic in practical settings. This is further highlighted by the fact that most lower-bound proofs rely on very carefully designed and fragile instances. In many real-world systems, partial information about future intervals is often available through historical trends or forecasting techniques, including machine-learned predictions. This motivates the study of \emph{learning-augmented algorithms} (also known as algorithms with predictions), which aim to bridge the gap between worst-case online guarantees and the performance of offline algorithms by leveraging predicted information during the decision process.

As in the traditional online setting, the standard formulation of interval scheduling requires decisions, that is acceptance or rejection of intervals, to be made irrevocably. 
However, it is known that no online algorithm can achieve a constant competitive ratio for this problem in general~\citep{Lipton1994OnlineIS, MSS-Long-1993}. To address this limitation, a relaxed model has been considered in which accepted intervals may be revoked prior to their deadlines, while rejections remain final. 
Several objective functions have been considered for the classic online interval scheduling problem, including maximizing the number of accepted intervals (unit-weight) and maximizing the total length of accepted intervals (proportional-weight). \cite{10.1007/978-3-031-38906-1_14} studies the unit-weight case, where the goal is to maximize the number of accepted intervals using predicted information. Moreover, \cite{karavasilis2025intervalselectionbinarypredictions} considers a proportional-weight objective, aiming to maximize the total length of accepted intervals in a revocable setting.

This leaves a natural gap: the irrevocable setting with proportional weights, which models many practical scenarios where job values depend on duration and decisions cannot be revoked. In this work, we fill this gap by providing a comprehensive analysis of learning-augmented interval scheduling under irrevocable decisions. We present both upper and lower bounds, characterizing the trade-offs between consistency and robustness in both deterministic and randomized settings. Moreover, we propose a smooth randomized algorithm and provide both theoretical guarantees and experimental results demonstrating its effectiveness on real-world data.

\subsection{Our Results and Techniques}

A common design pattern in learning-augmented algorithms is to combine two strategies: one that follows the predictions and one that ignores them entirely. This allows the algorithm to adapt its behavior based on the quality of the prediction. The challenge, of course, is deciding how and when to rely on each component.

A natural approach is to switch between a predictive and a classic algorithm based on the reliability of the prediction. This technique has been used effectively in several minimization problems~\citep{LykourisV21, KPZ18, rohatgi2020near, AntoniadisCEPS23, Wei20, NEURIPS2020_af94ed0d, DBLP:journals/corr/abs-2011-09076}, where the algorithm initially relies on the prediction and transitions to a backup strategy when the prediction appears unreliable, potentially continuing to switch as new evidence emerges. While switching in maximization problems poses additional challenges, it has been explored in online bipartite matching with imperfect advice~\citep{choo2024onlinebipartitematchingimperfect}. In this paper, we apply the switching technique to irrevocable online interval scheduling, where early commitments may permanently block future high-value intervals. Therefore, the timing of the switch, as well as the behavior in the prediction trusting phase, becomes particularly delicate.

While switching is a natural way to combine two strategies—one that trusts the prediction and one that ignores it—another powerful approach arises in the randomized setting: merging strategies probabilistically. We can define a randomized algorithm that selects between a predictive and a classic strategy based on a carefully chosen probability distribution. Leveraging this idea, we design a randomized algorithm that not only maintains robustness but also achieves \emph{smoothness}—its performance degrades gracefully as the prediction quality deteriorates.

We now summarize our main results and techniques, which build on these merging principles.

\paragraph{The \textsc{Trust-and-Switch} Framework.}
We introduce a general mechanism that allows algorithms to balance predictive and non-predictive behavior. The \textsc{Trust-and-Switch} framework takes as input a prediction model and a classic algorithm, and it decides—based on an online evaluation of the prediction—when to switch from trusting the prediction to relying on the classic algorithm. The switch occurs at most once and is irrevocable. We show that this simple structure captures the trade-off between consistency and robustness. More specifically, it achieves 1-consistency, and if the algorithm used after the switch is $\theta$-competitive, then the overall solution satisfies 
\[
|\opt| \le \theta \cdot |\alg| + 2k,
\]
where $k$ is the maximum interval length.

\paragraph{The \textsc{SemiTrust-and-Switch} Framework.}  
In addition to distinguishing between additive and multiplicative loss, our goal is to expose a tunable trade-off between consistency and robustness. This motivates a refined variant of the framework: instead of blindly following the prediction, the algorithm only follows the prediction for intervals whose length is below a threshold $\tau$, and accepts longer intervals greedily. This hybrid approach improves robustness without sacrificing too much consistency—in effect, it interpolates between fully trusting and fully ignoring the prediction. More specifically, it achieves $(1 + \frac{k}{\tau})$-consistency, and if the algorithm used after the switch is $\theta$-competitive, then the overall solution satisfies  
\[
|\opt| \le \max(\theta, \Delta + 1) \cdot |\alg| + \tau,
\]
where $\Delta$ is the ratio between the longest and shortest interval lengths.

\paragraph{Prediction Settings and Lower Bounds.}
Our frameworks work with a binary prediction model (denoted $\predO$), where predictions arrive online with the intervals and suggest whether to accept ($\hat{o}_i = 1$) or reject ($\hat{o}_i = 0$) interval $I_i$. We also establish lower bounds under a stronger, full-information setting (denoted $\predI$). Interestingly, our results show that having access to more detailed predictions, or receiving them offline, does not improve the achievable guarantees. In particular, for the \textsc{Trust-and-Switch} framework, we prove matching bounds for two-value instances: any $(1 + \alpha)$-consistent algorithm $\alg'$ with $\alpha < \frac{\lceil \Delta \rceil - 1}{\Delta}$ must incur a robustness loss of at least $\Delta \cdot |\alg'| + k$, even when given full predictions. This matches the guarantee of the framework for this class of instances, making the \textsc{Trust-and-Switch} framework optimal even when full predictions are available. Moreover, we provide a lower bound for our \textsc{SemiTrust-and-Switch} framework.

\paragraph{The \textsc{SmoothMerge} Algorithm.}
We introduce a randomized algorithm, \textsc{SmoothMerge}, that merges the behaviors of two strategies in a probabilistic manner. The algorithm simulates both strategies in parallel and independently. Upon the arrival of each interval, proceeds as follows: if the interval conflicts with previously accepted intervals, it is immediately rejected; otherwise, if both simulations agree on accepting (or rejecting), the algorithm follows that unanimous decision. When the simulations disagree, the algorithm accepts the interval with probabilities $p_t$ and $p_g$, corresponding to the two strategies.
This probabilistic merging yields both smoothness and robustness properties. Specifically, letting $\eta$ denote the prediction error, we have:
\begin{align*}
\mathbb{E}[|\alg|] \geq 
\max\Big\{ |\opt| \cdot (1 - \eta) \cdot p_t \cdot (1 - p_g), 
          \frac{p_g - p_t p_g}{1 - p_t p_g} \cdot |\greedy| \Big\}.
\end{align*}

\subsection{Related Work}

\paragraph{Online Interval Scheduling.}
The online interval scheduling problem was introduced by~\cite{Lipton1994OnlineIS}, who established the classical model in which intervals arrive sequentially and must be irrevocably accepted or rejected upon arrival. The goal is to maximize the total length of accepted, non-overlapping intervals. \cite{MSS-Long-1993} showed that there is no deterministic  algorithm that can achieve a constant competitive ratio, and \cite{Lipton1994OnlineIS} showed that no randomized algorithm can achieve a competitive ratio better than $O(\log \Delta)$, where $\Delta$ is the ratio between the longest and shortest intervals, and they provided an $O((\log \Delta)^{1+\epsilon})$-competitive algorithm. 

Several extensions of online interval scheduling have been studied. 
In the \emph{revocable} setting, where accepted intervals can be later removed, improved competitive guarantees are possible~\citep{10.1007/978-3-031-49815-2_13,GARAY1997180,TOMKINS1995173}.
Another extension of online interval scheduling is the \emph{any order} setting~\citep{10.1007/978-3-031-49815-2_13}, where the order of intervals can be arbitrary, and the \emph{sorted order} setting~\citep{Lipton1994OnlineIS}, where the intervals arrive in order of their release time.
In the \emph{weighted} setting, where intervals have arbitrary weights, no deterministic or randomized algorithm can achieve a finite competitive ratio in general~\citep{WOEGINGER19945, CanettiIrani1998}. However, for weight functions tied to interval length, \cite{WOEGINGER19945} identified classes (e.g., benevolent functions) that admit finite guarantees, and \cite{Seiden2000} gave improved randomized bounds. The complexity also varies with input structure: while disjoint path allocation is solvable in polynomial-time on trees and outerplanar graphs~\citep{DBLP:conf/icalp/GargVY93, 10.1007/3-540-57273-2_73}, it becomes NP-complete on general and series-parallel graphs~\citep{doi:10.1137/0205048, NISHIZEKI2001177}. 
For a comprehensive overview of the extensive literature on interval scheduling, we refer to the survey by \cite{KolenLenstraPapadimitriouSpieksma2007}, which provides thorough coverage of algorithmic techniques and complexity results across various problem variants.

\paragraph{Learning-augmented Algorithms.} Traditional worst-case analysis often fails to capture the practical performance of algorithms, motivating the study of beyond worst-case analysis~\citep{roughgarden2021beyond}. One prominent framework in this direction is that of learning-augmented algorithms (also known as algorithms with predictions), which incorporate machine-learned advice into decision-making~\citep{mitzenmacher2020algorithms}. \cite{LykourisV21} formalized this approach in the context of online caching, introducing the now-standard notions of consistency and robustness to measure how well an algorithm performs under both accurate and adversarial predictions.
This framework has since been applied to a wide range of online problems, including caching~\citep{LykourisV21, AntoniadisCEPS23, rohatgi2020near, Wei20}, secretary and matching problems~\citep{antoniadis2020secretary, dutting2021secretaries}, knapsack~\citep{im2021online}, Traveling Salesman Problem~\citep{Gouleakis_Lakis_Shahkarami_2023, bampis_et_al:LIPIcs.ESA.2023.12, chawla_et_al:LIPIcs.APPROX/RANDOM.2024.2}, 
online graph algorithms~\citep{azar2022online}, and various scheduling problems~\citep{KPZ18, NEURIPS2020_af94ed0d, mitzenmacher2020scheduling, lattanzi2020online, AJS22SWAT}. The scope of this area also extends beyond online problems. For an up-to-date list of papers in this field, we refer the reader to a curated online repository by~\cite{website}.

\paragraph{Online Interval Scheduling with Predictions.} 
Finally, the closest works to our own are~\cite{10.1007/978-3-031-38906-1_14} and~\cite{karavasilis2025intervalselectionbinarypredictions}.
\cite{10.1007/978-3-031-38906-1_14} study the unit-weight interval scheduling problem, where the objective is to maximize the number of accepted intervals, given predictions about the input sequence. They provide tight bounds on the competitive ratio achievable by deterministic algorithms. Their approach combines prediction-following with greedy acceptance when it is possible without a decrease in the profit of the planned solution. More recently,~\cite{karavasilis2025intervalselectionbinarypredictions} considered proportional-weight interval scheduling in a revocable setting, where the objective is to maximize the total length of accepted intervals and decisions can be revoked before deadlines. This work focuses on binary predictions and develops algorithms that can adaptively modify their decisions based on observed prediction quality. Our work addresses a crucial gap by studying proportional-weight interval scheduling with irrevocable decisions.

\section{Preliminaries}
\label{sec:prelim}
The input to the \emph{online interval scheduling problem} consists of a set $\I$ of $n$ intervals. Each interval $I_j \in \I$ is defined by a release time $r_j$ and a deadline $d_j$, with length $l_j = d_j - r_j$. Intervals arrive online: the information about interval $I_j$, namely its release time, deadline, and length, becomes available to the algorithm only at time $r_j$. Upon arrival, the algorithm must make an irrevocable decision to either accept or reject the interval, without knowledge of future intervals. The accepted intervals must be non-overlapping, and the objective is to maximize the total length of the accepted subset.

Throughout this work, we assume that intervals arrive in non-decreasing order of their release times. Without loss of generality, we label the intervals $I_1, I_2, \dots, I_n$ in order of arrival, so that $r_1 \leq r_2 \leq \dots \leq r_n$.

We define the following key parameters:
\begin{itemize}
    \item $k = \max_{I_j \in \I} l_j$: the maximum length of any interval.
    \item $\Delta = \frac{k}{\min_{I_j \in \I} l_j}$: the ratio between the maximum and minimum interval lengths.
\end{itemize}

Given an algorithm $\mathcal{A}$, we use $\mathcal{A}(\I)$ to denote the outcome of $\mathcal{A}$ on instance $\I$, and $|\mathcal{A}(\I)|$ to denote the total length of intervals selected by it.

We define notation for subsets of intervals based on their release times and deadlines. For timepoints $t_1, t_2 \in \mathbb{R}_{\ge 0}$, let $\I(t_1, t_2)$ denote the subset of intervals in $\I$ whose release times and deadlines satisfy constraints determined by $t_1$ and $t_2$, respectively. 
We use the symbol $\cdot$ to indicate that no constraint is applied to that component (release time or deadline). Superscripts $\leftarrow$ and $\rightarrow$ indicate the direction of the inequality: for a timepoint $t$, $t^\leftarrow$ denotes times at or before $t$ (i.e., $\leq t$), and $t^\rightarrow$ denotes times strictly after $t$ (i.e., $> t$).
For example:
\begin{itemize}
    \item $\I(t_1^\leftarrow, \cdot)$ is the set of intervals with release time $\leq t_1$ (no constraint on deadline),
    \item $\I(\cdot, t_2^\leftarrow)$ is the set of intervals with deadline $\leq t_2$ (no constraint on release time),
    \item $\I(t_1^\leftarrow, t_2^\rightarrow)$ is the set of intervals with release time $\leq t_1$ and deadline $> t_2$.
\end{itemize}

We denote the optimal offline algorithm by $\opt$. In particular, $|\opt(\I)|$ denotes the total length of the optimal offline solution, and $|\alg(\I)|$ denotes the total length achieved by the online algorithm under consideration.

The performance of an online algorithm is evaluated using the \textit{competitive ratio}, which measures the worst-case performance relative to the optimal offline solution. The competitive ratio is defined as:
\[
\text{competitive ratio} = \sup_{\I} \frac{|\opt(\I)|}{|\alg(\I)|},
\]
where the supremum is taken over all possible input instances $\I$.
An online algorithm is said to be \textit{$c$-competitive} if $\text{competitive ratio} \leq c$. In this case, the total length of intervals accepted by the algorithm is guaranteed to be at least $1/c$ of the total length of intervals accepted by the optimal offline solution for any input instance.

For randomized algorithms, the competitive ratio is defined based on the expected performance of the algorithm. Let $\mathbb{E}[\alg(\I)]$ denote the expected total length of intervals accepted by the randomized online algorithm. The competitive ratio for randomized algorithms is defined as:
\[
\text{competitive ratio (randomized)} = \sup_{\I} \frac{|\opt(\I)|}{\mathbb{E}[\alg(\I)]}.
\]
A randomized online algorithm is said to be \textit{$c$-competitive} if $\text{competitive ratio (randomized)} \leq c$, meaning that the expected total length of intervals accepted by the algorithm is at least $1/c$ of the total length accepted by the optimal offline solution, for any input instance $\I$.

\paragraph*{Prediction Setup.}
Learning-augmented algorithms represent a class of algorithms that leverage pre-computed predictions to enhance decision-making in online settings. These predictions, often derived from sources such as statistical models or machine learning systems, provide guidance to the algorithm without requiring it to learn from real-time data. In the context of the online interval scheduling problem, we consider the following prediction settings:
\begin{itemize}
    \item \textbf{Predictions $\predO$:} The predictions $\predO = \{\hat{o}_1, \hat{o}_2, \dots, \hat{o}_n\}$ are a binary vector, where $\hat{o}_i \in \{0, 1\}$ for each interval $I_i \in \I$. The value $\hat{o}_i = 1$ indicates that the interval $I_i$ should be accepted, while $\hat{o}_i = 0$ indicates that the interval should be rejected. This prediction represents a suggested optimal solution.
    
    \item \textbf{Predictions $\predI$:} The predictions $\predI$ aim to provide complete foresight into the characteristics of all intervals in the input. Specifically, $\predI = \{(\hat{r}_1, \hat{d}_1), (\hat{r}_2, \hat{d}_2), \dots, (\hat{r}_n, \hat{d}_n)\}$, where each tuple $(\hat{r}_i, \hat{d}_i)$ contains the predicted release time $\hat{r}_i$ and deadline $\hat{d}_i$ of interval $I_i$. The length of each interval can be computed as $\hat{l}_i = \hat{d}_i - \hat{r}_i$. This prediction setting aims to provide the algorithm with more detailed information to guide its decisions.
    
\end{itemize}

\paragraph*{Performance Metrics.}

We use $\predGen$ to denote a generic prediction model, which can represent any of the specific prediction settings considered in this work ($\predO$ or $\predI$). If $\I \vartriangleleft \predGen$ denotes all instances $\I$ for which the prediction $\predGen$ is accurate, an algorithm is 
\begin{itemize}
    \item $\alpha$-\textit{consistent} if it is $\alpha$-\textit{competitive} when the prediction is correct, i.e.,
    \begin{equation*}
    \max_{\I, \predGen : \I \vartriangleleft \predGen} \left\{ \frac{|\text{OPT}(\I)|}{|\text{ALG}(\I, \predGen)|} \right\} \leq \alpha,
    \end{equation*}
    where $\text{ALG}(\I, \predGen)$ denotes the solution returned by the algorithm using the prediction $\predGen$, and $\text{OPT}(\I)$ denotes the optimal offline solution.

    \item $\beta$-\textit{robust} if it is $\beta$-\textit{competitive} regardless of the quality of the prediction, i.e.,
    \begin{equation*}
    \max_{\I, \predGen} \left\{ \frac{|\text{OPT}(\I)|}{|\text{ALG}(\I, \predGen)|} \right\} \leq \beta,
    \end{equation*}
    where $\I$ denotes all possible input instances, and the algorithm's performance is bounded even when the prediction $\predGen$ is arbitrarily inaccurate.

    We denote by $\pred$ the algorithm that follows the prediction blindly. Given a prediction $\predO$, for each interval $I_j$, it accepts $I_j$ if and only if $\hat{o}_j = 1$.
    
\end{itemize}

\section{Warm-up: \textsc{Trust-and-Switch} Framework}
In this section, we introduce a general framework for addressing the online interval scheduling problem with predictions. Roughly speaking, the idea is to begin by following the predictions and fall back to a classic algorithm once it is certain that the predictions are not perfectly accurate. Using this framework, we design both deterministic and randomized algorithms that utilize predictions $\predO$ about the optimal set of intervals.

At time $r_j$, the algorithm has access to the set $\I(r_j^\leftarrow, \cdot)$, that is, all intervals released up to that point, alongside with the corresponding predictions for these intervals.

\paragraph{Framework \textsc{Trust-and-Switch}.} 
The general framework receives intervals and predictions in an online manner, along with a classic algorithm as input, and returns a robust consistent online algorithm. It consists of two phases: the \firstPhase\ and the \secondPhase. The algorithm begins in the \firstPhase, where it follows the predictions blindly and may transition to the \secondPhase\ at a designated time point if the predictions are found to be inaccurate. In the \firstPhase, decisions on whether to accept or reject an interval are made solely based on the prediction. In contrast, the \secondPhase\ ignores the predictions entirely and makes decisions using the given classic algorithm, without relying on any prediction.

Upon the arrival of interval $I_j$ at time $r_j$, the framework evaluates the predictions, and the moment it is certain that the predictions are inaccurate, it decides to switch to the classic algorithm. Although the framework re-evaluates the accuracy of the predictions at each release time $r_j$, it also takes into account the whole of $\mathcal{I}(r_j^\leftarrow,\cdot)$.

More formally, upon the arrival of interval $I_j$ at time $r_j$, the framework computes the \emph{evaluation point}:
\[
t_j := \max\left( \{r_j\} \cup \{d_i \mid I_i \in \I(r_j^\leftarrow, \cdot),\ \hat{o}_i = 1\} \right).
\]
This value ensures that if there is an accepted interval predicted to be in $\predO$ that is still active at time $r_j$, we evaluate the prediction up to the last deadline among predicted intervals in $\I(r_j^\leftarrow, \cdot)$. Otherwise, we simply consider $r_j$ as the evaluation point.

Note that if $\hat{o}_j = 1$, then the corresponding evaluation point is $t_j = d_j$; however, if we are already certain at $r_j$ about the inaccuracy of the predictions, the switch to the \secondPhase\ can take place earlier.

To evaluate the accuracy of the prediction up to the time $t_j$, the algorithm first verifies the consistency of $\predO$: in other words, whether all intervals in $\I(r_j^\leftarrow, \cdot)$ with $\hat{o}_i = 1$, including $I_j$, are non-overlapping. If not, the prediction $\predO$ is clearly invalid, and the algorithm switches as soon as possible to the \secondPhase.

Otherwise, it compares the quality of the prediction with the offline optimum up to time $t_j$. Specifically, it verifies whether
\[
|\opt(\I(r_j^\leftarrow, t_j^\leftarrow))| > |\pred(\I(r_j^\leftarrow, t_j^\leftarrow))|.
\]

If this inequality holds, the algorithm transitions to the \secondPhase\ and starts to run the classic algorithm as soon as possible -- that is: at $r_j$ if $I_j \in \predO$, and at $t_j$ otherwise -- on the instance $\mathcal{I}(r_j^\rightarrow,\cdot)$ (resp. $\mathcal{I}(t_j^\rightarrow,\cdot)$) consisting of the intervals released after this point. Otherwise, it continues in the \firstPhase\ and accepts $I_j$ if and only if $\hat{o}_j = 1$.

\begin{figure}[t]
    \centering
    \begin{tikzpicture}[scale=1]
        \foreach \y/\xstart/\xend in {
            1/4/10          
        } {
            \drawInterval[MyBlue, thick]{solid}{\y}{\xstart}{\xend};
        }
        \foreach \y/\xstart/\xend in {
            2.2/1/1.8,
            1.8/0.3/1.2,
            1.8/1.4/2.3,
            1.8/2.5/3.2,
            1.4/0.5/5
        } {
            \drawInterval[MyRed, thick]{densely dotted}{\y}{\xstart}{\xend};
        }
        
        \foreach \y/\xstart/\xend in {
            2.2/2/6,
            1.8/4.5/11,
            1.4/6/7
        } {
            \drawInterval[MyBlue, thick]{solid}{\y}{\xstart}{\xend};
        }

        \foreach \y/\xstart/\xend in {
            2.2/6.5/7,
            2.2/7.5/10,
            1.8/11.2/12,
            1.4/7.5/8.5,
            1.4/9/10.2,
            1.4/10.4/12
        } {
            \drawInterval[MyGreen, thick]{densely dashed}{\y}{\xstart}{\xend};
        }
        
        \draw [gray, thick, dotted] (6, 0) -- (6, 3.5);
        % \draw [gray, thick, dotted] (10, 0) -- (10, 3);

        \node [gray] at (4, 0.7) {$r_j$};
        \node [gray] at (10, 0.7) {$d_j$};
        \node [gray] at (2, 2.5) {$r_i$};
        \node [gray] at (6, 2.5 ) {$d_i$};
        \node [black] at (4, 2.7) {$I_i \in \predO$};
        \node [gray] at (6, 0) {$t_j = d_i$};

        \node [MyRed] at (2, 0.5) {$\I_1$};
        \node [MyBlue] at (7,0.5) {$\I_2$};
        \node [MyGreen] at (11, 0.5) {$\I_3$};
                
    \end{tikzpicture}
    \caption{Classification of intervals when, at timepoint $r_j$, the framework decides to switch based on the evaluation up to $t_j = d_i$, where $I_i \in \predO$. The three classes $\mathcal{I}_1$, $\mathcal{I}_2$, and $\mathcal{I}_3$ are illustrated with distinct visual styles in the figure.}
    \label{fig:prediction}
\end{figure}
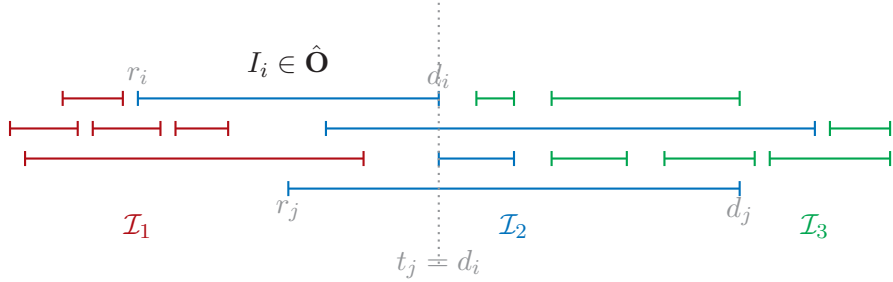

We first show that whenever the framework switches to the \secondPhase, the predicted solution is indeed suboptimal. This implies that whenever the predictions are accurate, the framework remains in the \firstPhase; this guarantees the $1$-consistency of the framework.

\begin{lemma}\label{lem:Framewok-1-cons}
The \textsc{Trust-and-Switch} framework satisfies $1$-consistency.
\end{lemma}

\begin{proof}
    We want to show that if the predictions are accurate, the framework never switches which directly implies its $1$-consistency.
    
    Assume for contradiction that the prediction $\predO$ is accurate, but the framework switches at time $t_j$. By construction, this means that
    \[
    |\opt(\I(r_j^\leftarrow, t_j^\leftarrow))| >  |\pred(\I(r_j^\leftarrow, t_j^\leftarrow))|.
    \]
    
    Since $t_j$ is either a release time or deadline of some interval in $\predO$, or the release time of an interval not in $\predO$ that does not overlap with any predicted interval, we can partition the predicted intervals into two disjoint subsets:
    \begin{itemize}
        \item $\I_{\text{before}} := \{ I_i \in \I(r_j^\leftarrow, t_j^\leftarrow) \mid \hat{o}_i = 1 \}$, i.e., predicted intervals released by $r_j$, that complete at or before $t_j$,
        \item $\I_{\text{after}} := \{ I_i \in \I(t_j^\rightarrow, \cdot) \mid \hat{o}_i = 1 \}$, i.e., predicted intervals released after $t_j$.
    \end{itemize}
    
    Now consider constructing a new feasible solution by replacing $\I_{\text{before}}$ with the intervals in $\opt(\I(r_j^\leftarrow, t_j^\leftarrow))$ and keeping $\I_{\text{after}}$ unchanged (note that $\I_\text{after}$ contains no two intersecting intervals, by the assumption on the optimality of the prediction, and furthermore does not contain any interval intersecting with $\I(r_j^\leftarrow, t_j^\leftarrow)$).  
    The total length of this new solution is
    \[
    |\opt(\I(r_j^\leftarrow, t_j))| + \sum_{I_i \in \I_{\text{after}}} l_i > \sum_{I_i \in \I_{\text{before}}} l_i + \sum_{I_i \in \I_{\text{after}}} l_i = |\pred(\I)|,
    \]
    contradicting the optimality of the prediction.
    
    Therefore, the framework does not switch when the predictions are accurate, and thus guarantees $1$-consistency.
\end{proof}

We now analyze the robustness of the \textsc{Trust-and-Switch} framework. Specifically, we show that even in the presence of inaccurate predictions, the total value of the optimal solution is at most $\theta \cdot |\alg| + 2k$, where the additive loss accounts for following inaccurate predictions in the \firstPhase\ and the transition to the \secondPhase, and the multiplicative term corresponds to the use of a $\theta$-competitive algorithm in the \secondPhase.

\begin{theorem}\label{Thm:Framewok-robustness}
Consider the \textsc{Trust-and-Switch} framework $\alg$, and assume that the \secondPhase\ uses a $\theta$-competitive algorithm. Then the algorithm satisfies
\[
|\opt| \le \theta \cdot |\alg| + 2k.
\]
\end{theorem}

\begin{proof}
Assume that the framework decides to switch to the \secondPhase\ at time $r_j$, based on the evaluation of the prediction up to $t_j$, and we have to wait until then to be able to switch. If the framework never switches, we set $j = n$.

To analyze the performance of the algorithm, we partition the intervals into three disjoint classes based on their relation to the switching point $t_j$ (see Figure~\ref{fig:prediction}):
\begin{itemize}
    \item $\I_1 := \I(\cdot, t_j^\leftarrow)$, intervals with both release time and deadline strictly before $t_j$,
    \item $\I_2 := \I(t_j^\leftarrow, t_j^\rightarrow)$, intervals with release time at or before $t_j$ and deadline strictly after $t_j$,
    \item $\I_3 := \I(t_j^\rightarrow, \cdot)$, intervals with release time strictly after $t_j$.
\end{itemize}

We now analyze the contribution of each class to the overall value of the optimal solution and the algorithm, beginning with Class $\I_1$.

Note that $j > 1$, since the algorithm does not decide to switch at $t_1$, regardless of whether $\hat{o}_1 = 1$ or $\hat{o}_1 = 0$. We distinguish two cases based on how $t_j$ is determined:

\begin{itemize}
    \item \textbf{Case 1:} $t_j = d_i$ for some $i = \arg\max\{d_i \mid I_i \in \I(r_j^\leftarrow, \cdot),\ \hat{o}_i = 1\}$ and $i \neq j$.

        In this case, $t_i = t_j = d_i$. Since the framework did not switch at $r_i$, it must be that:
        \begin{align*}
            |\opt(\I(r_i^\leftarrow, d_i^\leftarrow))| &\leq |\alg(\I(r_i^\leftarrow, d_i^\leftarrow))| \\ 
            &= |\pred(\I(r_i^\leftarrow, d_i^\leftarrow))|.
            \end{align*}

        Since $\alg$ accepted interval $I_i$, and we decide to switch at $r_j$ with $r_i < r_j < d_i$, we have:
        \[
        |\alg(\I(r_j^\leftarrow, d_i^\leftarrow))| = |\alg(\I(r_i^\leftarrow, d_i^\leftarrow))|.
        \]

        Meanwhile, the optimal offline solution may include additional intervals $w$ with $r_i < r_w < d_w < d_i = t_j$, whose total length is at most $l_i$. Thus:
        \[
        |\opt(\I(r_j^\leftarrow, d_i^\leftarrow))| < |\opt(\I(r_i^\leftarrow, d_i^\leftarrow))| + l_i.
        \]

        Therefore,
        \[
        |\opt(\I(r_j^\leftarrow, t_j^\leftarrow))| < |\alg(\I(r_j^\leftarrow, t_j^\leftarrow))| + k.
        \]
   
        \item \textbf{Case 2:} $t_j = d_j$ and $\hat{o}_j = 1$.

        In this case, we have $t_j = d_j$. Consider the prefix $\I(r_{j-1}^\leftarrow, \cdot)$. If $I_j$ overlaps with a previously predicted interval, we can bound the loss using the same argument as in Case~1. Therefore, we assume that the predicted intervals do not conflict with each other. In this case, we must have $t_{i-1} < r_i$. Since the framework did not switch at $t_{i-1}$, we have:
        \begin{align*}
        |\opt(\I(r_{i-1}^\leftarrow, t_{i-1}^\leftarrow))| &\leq |\alg(\I(r_{i-1}^\leftarrow, t_{i-1}^\leftarrow))| \\
        &= |\pred(\I(r_{i-1}^\leftarrow, t_{i-1}^\leftarrow))|.
        \end{align*}

        Since $\pred$ accepts interval $I_i$, we have:
        \[
        |\alg(\I(r_i^\leftarrow, t_i^\leftarrow))| = |\alg(\I(r_{i-1}^\leftarrow, t_{i-1}^\leftarrow))| + l_i.
        \]

        On the other hand, the optimal offline solution may include one additional interval $w$ with $r_w \leq r_{i-1}$ and $t_{i-1} < d_w < t_i$. Its length cannot exceed $k$, so:
        \[
        |\opt(\I(r_i^\leftarrow, t_i^\leftarrow))| < |\opt(\I(r_{i-1}^\leftarrow, t_{i-1}^\leftarrow))| + k.
        \]

        Therefore,
        \[
        |\opt(\I(r_i^\leftarrow, t_i^\leftarrow))| < |\alg(\I(r_i^\leftarrow, t_i^\leftarrow))| + (k - l_i).
        \]
        
        \item \textbf{Case 3:} $t_j = r_j$.

In this case, $t_j$ is determined by the release time of $I_j$, which implies that no predicted interval with $\hat{o}_i = 1$ in $\I(r_j^\leftarrow, \cdot)$ overlaps $I_j$. Hence, $t_{j-1} < t_j = r_j$. Because the framework did not switch at $t_{j-1}$, we must have
\begin{align*}
    |\opt(\I(r_{j-1}^\leftarrow, t_{j-1}^\leftarrow))| &\leq |\alg(\I(r_{j-1}^\leftarrow, t_{j-1}^\leftarrow))| \\
    &= |\pred(\I(r_{j-1}^\leftarrow, t_{j-1}^\leftarrow))|.
\end{align*}

At time $t_j$, the framework does switch, so
\[
|\opt(\I(r_j^\leftarrow, t_j^\leftarrow))| > |\pred(\I(r_j^\leftarrow, t_j^\leftarrow))|.
\]

The only difference between $\I(r_j^\leftarrow, t_j^\leftarrow)$ and $\I(r_{j-1}^\leftarrow, t_{j-1}^\leftarrow)$ is the set of intervals $w$ whose release time is at or before $t_{j-1}$ and whose deadline lies in the open interval $(t_{j-1}, r_j)$. More formally,

\[
\I(r_j^\leftarrow, t_j^\leftarrow) \setminus \I(r_{j-1}^\leftarrow, t_{j-1}^\leftarrow) = \I(t_{j-1}^\leftarrow,\ (t_{j-1}, r_j)).
\]

Because all such intervals contain the point $t_{j-1}$, the non-overlapping constraint implies that the optimal solution can select at most one of them. Consequently,
\[
|\opt(\I(r_j^\leftarrow, t_j^\leftarrow))| \le |\opt(\I(r_{j-1}^\leftarrow, t_{j-1}^\leftarrow))| + k.
\]

Therefore, since $\alg$ does not accept any new interval in the time window $(t_{j-1}, t_j)$, we have $|\alg(\I(r_j^\leftarrow, t_j^\leftarrow))| = |\alg(\I(r_{j-1}^\leftarrow, t_{j-1}^\leftarrow))|$, and thus
\[
|\opt(\I(r_j^\leftarrow, t_j^\leftarrow))| \le |\alg(\I(r_j^\leftarrow, t_j^\leftarrow))| + k.
\]
\end{itemize}

So far, we have bounded the cost of the algorithm for the first class of intervals with deadlines before $t_j$. In both cases, we have shown that the value of the optimal solution for intervals in $\I_1$ is at most the value obtained by the algorithm on $\I_1$ plus an additive term of $k$:
\[
|\opt(\I_1)| \le |\alg(\I_1)| + k.
\]

We now turn to the second class of intervals, $\I_2$. If the algorithm switches to the \secondPhase\ at time $t_j$, where $t_j = d_i$ for some $j \neq i = \arg\max\{d_i \mid I_i \in \I(r_j^\leftarrow, \cdot),\ \hat{o}_i = 1\}$, then there may be a set of intervals whose release time is before $t_j$ and whose deadline is after $t_j$. These intervals are not considered in either phase of the algorithm. This set might be non-empty only in Case~1.

Since all of these intervals contain the timepoint $t_j$, the non-overlapping constraint implies that the optimal solution can include at most one of them. Moreover, since the algorithm does not select any interval in this class, we have:
\[
|\opt(\I_2)| \le |\alg(\I_2)| + k = k.
\]

Finally, we consider the third class $\I_3$, consisting of intervals with release time at or after $t_j$. By construction, the algorithm switches to the \secondPhase\ at time $t_j$ and runs a $\theta$-competitive algorithm on this suffix of the instance. Therefore, we have:
\[
|\opt(\I_3)| \le \theta \cdot |\alg(\I_3)|.
\]

Since intervals in different classes might overlap with each other, we can bound the optimal value as:
\[
|\opt| \leq |\opt(\I_1)| + |\opt(\I_2)| + |\opt(\I_3)|.
\]

Putting all parts together,
\begin{align*}
|\opt| &\leq |\opt(\I_1)| + |\opt(\I_2)| + |\opt(\I_3)| \\
       &\le (|\alg(\I_1)| + k) + k + \theta \cdot |\alg(\I_3)|,
\end{align*}
and since $\theta > 1$, we conclude
\[
|\opt| \leq \theta \cdot |\alg| + 2k.
\]

This completes the proof.
\end{proof}

We note that the framework allows any deterministic or randomized algorithm to be used in the \secondPhase. In practice, it makes sense to apply the best-known classical algorithm for the given setting, whether deterministic or randomized.

\subsubsection{Two-Value Instances}
We give a tighter bound for two-value instances. All intervals in these instances have one of two distinct lengths. We refer to these as \emph{short} and \emph{long} intervals. More specifically, we show that the \textsc{Trust-and-Switch} framework provides a better robustness guarantee in this special case, due to the restricted structure of the interval lengths.

\begin{lemma}\label{lem:TrustFramewok-2value-robustness}
Consider the \textsc{Trust-and-Switch} framework $\alg$, and assume that the \secondPhase\ uses a $\theta$-competitive algorithm. Then, on two-value instances, the algorithm satisfies
\[
|\opt| \le \max\{\theta, 2\} \cdot |\alg| + k.
\]
\end{lemma}
\begin{proof}
The proof follows the same structure as the proof of Theorem~\ref{Thm:Framewok-robustness}. Assume that the framework decides to switch to the \secondPhase\ at time $r_j$, based on the evaluation of the prediction up to $t_j$, and we have to wait until then to be able to switch. If the framework never switches, we set $j = n$.

To analyze the performance of the algorithm, we partition the intervals into three disjoint classes based on their relation to the switching point $t_j$:
\begin{itemize}
    \item $\I_1 := \I(\cdot, t_j^\leftarrow)$, intervals with both release time and deadline strictly before $t_j$,
    \item $\I_2 := \I(t_j^\leftarrow, t_j^\rightarrow)$, intervals with release time at or before $t_j$ and deadline strictly after $t_j$,
    \item $\I_3 := \I(t_j^\rightarrow, \cdot)$, intervals with release time strictly after $t_j$.
\end{itemize}

As before, the analysis for intervals in $\I_3$ remains unchanged. Our goal is to analyze the intervals in $\I_1$ and $\I_2$ together.

We again distinguish the performance of the algorithm based on how the evaluation point $t_j$ is determined. We consider the following three cases:
\begin{itemize}
    \item \textbf{Case 1:} $t_j = d_i$ for some $i \ne j$, where $I_i \in \predO$ and $d_i = \max\{d_i \mid I_i \in \I(r_j^\leftarrow, \cdot),\ \hat{o}_i = 1\}$,
    \item \textbf{Case 2:} $t_j = d_j$, the deadline of the current interval $I_j$ (since $\hat{o}_j = 1$),
    \item \textbf{Case 3:} $t_j = r_j$, the evaluation point is the release time of $I_j$.
\end{itemize}

Since $\I_2 = \emptyset$ in Cases~2 and~3, the analysis remains unchanged in those cases, and we lose at most an additive factor of $k$:
\[
|\opt(\I_1 \cup \I_2)| \le |\alg(\I_1 \cup \I_2)| + k.
\]

However, we now provide a refined analysis for Case~1 by jointly analyzing the intervals in Classes $\I_1$ and $\I_2$ to establish the following bound:
\[
|\opt(\I_1 \cup \I_2)| \le 2 \cdot |\alg(\I_1 \cup \I_2)| +k.
\]

In this case, we have $t_j = d_i$ for some $i \ne j$. We distinguish two cases based on the length $l_i$.

If $\I_i$ is a short interval, then since the algorithm did not decide to switch at $r_i$, we have:
\[
|\opt(\I(r_i^\leftarrow, d_i^\leftarrow))| \leq |\alg(\I(r_i^\leftarrow, d_i^\leftarrow))| = |\pred(\I(r_i^\leftarrow, d_i^\leftarrow))|.
\]

Moreover, since $\I_i$ is short, we have:
\[
|\opt(\I(r_i^\leftarrow, d_i^\leftarrow))| = |\opt(\I(r_j^\leftarrow, d_i^\leftarrow))| \quad \text{and} \quad |\alg(\I(r_i^\leftarrow, d_i^\leftarrow))| = |\alg(\I(r_j^\leftarrow, d_i^\leftarrow))|.
\]

Therefore, we still have:
\[
|\opt(\I(r_j^\leftarrow, d_i^\leftarrow))| \leq |\alg(\I(r_j^\leftarrow, d_i^\leftarrow))|,
\]
and the algorithm does not decide to switch.

If $l_i = k$, then since the algorithm did not decide to switch at $r_i$, we have:
\[
|\opt(\I(r_i^\leftarrow, d_i^\leftarrow))| \leq |\alg(\I(r_i^\leftarrow, d_i^\leftarrow))| = |\pred(\I(r_i^\leftarrow, d_i^\leftarrow))|.
\]

This implies:
\[
|\opt(\I(r_i^\leftarrow, d_i^\leftarrow))| \leq |\alg(\I(r_i^\leftarrow, r_i^\leftarrow))| + k.
\]

Now we consider the long interval $I_i$ selected by $\alg$, together with all intervals in $\opt$ whose release times lie in $\left(r_i, d_i\right]$, that is, intervals that $\alg$ cannot select due to a conflict with $I_i$.

Among these conflicting intervals, there can be at most one long interval and at most $\lceil \Delta \rceil - 1$ short intervals. Therefore, the total length of these intervals is bounded above by:
\[
k + \left(\lceil \Delta \rceil - 1\right) \cdot \frac{k}{\Delta},
\]
and so their contribution relative to $k$ is:
\[
\frac{k + (\lceil \Delta \rceil - 1)\cdot \frac{k}{\Delta}}{k} = 1 + \frac{\lceil \Delta \rceil - 1}{\Delta} < 2.
\]

Therefore, we can bound the optimal value on $\I_1 \cup \I_2$ as:
\begin{align*}
|\opt(\I_1 \cup \I_2)| &= |\opt(\I(r_i^\leftarrow, d_i^\leftarrow))| + |\opt(\I(\left(r_i, d_i\right], \cdot))| \\
&\leq |\alg(\I(r_i^\leftarrow, r_i^\leftarrow))| + k + 2 \cdot |\alg(\I(r_i^\rightarrow, d_i^\leftarrow))| \\
&\leq 2 \cdot |\alg(\I_1 \cup \I_2)| + k.
\end{align*}

Considering all cases, we have shown that the value of the optimal solution for intervals in $\I_1 \cup \I_2$ is at most
\[
|\opt(\I_1 \cup \I_2)| \le 2 \cdot |\alg(\I_1 \cup \I_2)| + k.
\]

Since intervals in different classes might overlap with each other, we can bound the optimal value as:
\[
|\opt| \leq |\opt(\I_1 \cup \I_2)| + |\opt(\I_3)|.
\]

Putting all parts together,
\begin{align*}
|\opt| &\leq |\opt(\I_1 \cup \I_2)| + |\opt(\I_3)| \\
       &\le 2 \cdot |\alg(\I_1 \cup \I_2)| + k + \theta \cdot |\alg(\I_3)|.
\end{align*}

Since $\theta > 1$, we conclude
\[
|\opt| \leq \max\{\theta, 2\} \cdot |\alg| + k.
\]

This completes the proof.
\end{proof}

\subsection{Tightness}
In this final subsection, we show that the guarantees of the \textsc{Trust-and-Switch} framework are tight for two-value instances. First, we prove that the \greedy\ algorithm achieves a robustness guarantee in this setting. Then, we show that even under the strongest prediction model $\predI$, no algorithm with consistency better than $1 + \frac{\lceil \Delta \rceil - 1}{\Delta}$ can achieve a stronger robustness guarantee. This establishes the tightness of the \textsc{Trust-and-Switch} framework for two-value instances.
While our lower bound extends to general instances as well, there remains a gap between the upper and lower bounds in the general setting.

\subsubsection{\greedy\ Algorithm}

We first bound the competitive ratio of the \greedy\ algorithm on two-value instances. This will allow us to instantiate the robustness guarantee in our framework. Specifically, we show that \textsc{Trust-and-Switch}(\greedy) satisfies $1$-consistency and guarantees
\[
|\opt| \le \max\{\Delta, 2\} \cdot |\textsc{Trust-and-Switch}(\greedy)| + k.
\]
We briefly recall the definition of the \greedy\ algorithm below.

\paragraph{\greedy.}
The \greedy\ algorithm accepts an arriving interval if and only if it does not overlap with any previously accepted interval.

\begin{lemma}\label{lemma:greedyub}
    Given a two-value instance $\I$, let $\greedy(\I)$ be the solution obtained by the \greedy\ algorithm. Then,
    \[
    |\opt(\I)| \leq \max\{\Delta, 2\} \cdot |\greedy(\I)|.
    \]
\end{lemma}
\begin{proof}
    We first note that by the definition of $\greedy$ for each interval $\I_i\in\opt(\I)$ there exists exactly one interval $I_j\in\greedy(\I)$ with $r_j\le r_i$ and $I_i\cap I_j\neq \varnothing$  (note that it could be that $I_i=I_j$). We define $f:\opt(\I)\rightarrow \greedy(\I)$ such that for each $I_i\in\opt(\I)$, $f(I_i)$ be exactly the interval $I_j\in\greedy(\I)$ as described above.

    Note that $f$ is many-to-one and furthermore for each interval $I_j\in \greedy(\I)$, $f^{-1}(I_j)$ can consist of at most one interval if $I_j$ is short (and thus $|f^{-1}(I_j)|\le \Delta l_s$ in this case), and $\sum_{I_i \in f^{-1}(I_j)} |I_i| < 2\cdot k$ if $I_j$ is long -- since  at most one interval of $f^{-1}(I_j)$ can be long. This gives

    \begin{align*}
    |\opt(\I)| &= 
    \sum_{I_i\in\opt(\I)} |I_i|\\
    &= 
    \sum_{\substack{I_i\in\opt(\I)\\ f(I_i) \text{ short}}} |I_i| + \sum_{\substack{I_i\in\opt(\I)\\ f(I_i) \text{ long}}} |I_i|\\
    &\le \Delta \cdot \sum_{\substack{I_i\in\opt(\I)\\ f(I_i) \text{ short}}} |f(I_i)| + 2\cdot \sum_{\substack{I_i\in\opt(\I)\\ f(I_i) \text{ long}}} |f(I_i)|\\
    &\le \max\{2,\Delta\} |\greedy(\I)|.
    \end{align*}
\end{proof}

Now, using Theorem~\ref{Thm:Framewok-robustness}, and Lemmas~\ref{lem:TrustFramewok-2value-robustness}, and~\ref{lemma:greedyub}, we can conclude the following proposition.

\begin{proposition}
Given a two-value instance and predictions $\predO$, \textsc{Trust-and-Switch}(\greedy) satisfies $1$-consistency. Moreover, it guarantees
\[
|\opt| \le \max\{\Delta, 2\} \cdot |\textsc{Trust-and-Switch}(\greedy)| + k,
\]
independent of the quality of the predictions.
\end{proposition}

\subsubsection{Lower Bound}

We now show that the consistency–robustness trade-off achieved by the \textsc{Trust-and-Switch} framework is tight for two-value instances. Even with access to full information through $\predI$, no algorithm with sufficiently strong consistency can guarantee a better robustness bound.

\begin{lemma}\label{lem:LB-2value}
For any $(1 + \alpha)$-consistent algorithm $\alg'$ for two-value instances, with $\alpha < \frac{\lceil \Delta \rceil - 1}{\Delta}$, there exists an instance $\I$ such that
\[
\Delta \cdot |\alg'(\I)| + k \le |\opt(\I)|,
\]
even when provided with full predictions $\predI$.
\end{lemma}

\begin{proof}
We construct an adversarial input sequence. Note that this instance can be constructed even if the total number of intervals is known in advance. The key idea is to present the algorithm with a long interval that is not predicted by $\predO$. If this interval is accepted, consistency is violated by introducing many short intervals and an additional long interval that do not conflict with each other but do conflict with the accepted interval. Conversely, if the long interval is rejected, the algorithm immediately incurs a loss of $k$, assuming that the predictions are inaccurate.

We begin by describing the predicted instance $\predI$ that the algorithm receives a priori. We then define the actual instance $\Iadversary$, which is constructed adaptively based on the decisions made by the algorithm $\alg'$.

\paragraph{Predicted Instance $\predI$.} 
$\predI$ consists of $\lceil \Delta \rceil^2 + 2$ intervals, indexed in non-decreasing order of release time. These intervals are defined recursively as follow:

\begin{itemize}
    \item $I_1 = \left[0,\ k\right)$ — a long interval,
    \item $I_2 = \left[\epsilon,\ \epsilon + \frac{k}{\Delta} \right)$ — the first short interval,
    \item $I_j = \left[d_{j-1} + \epsilon,\ d_{j-1} + \epsilon + \frac{k}{\Delta} \right)$ for $j = 3,\dots,\lceil \Delta \rceil$ — a chain of short intervals following $I_2$,
    \item $I_j = \left[k - \epsilon \cdot \frac{k}{\Delta},\ k - \epsilon \right)$ for $j = \lceil \Delta \rceil + 1,\dots,\lceil \Delta \rceil^2 + 1$ — a block of overlapping short intervals packed at the end of the long interval $I_1$,
    \item $I_n = \left[k - \epsilon,\ 2k - \epsilon \right)$ — the final long interval that touches the end of $I_1$.
\end{itemize}

Here, $\epsilon > 0$ is chosen such that 
\[
\epsilon < 1 + \frac{1 - \lceil \Delta \rceil}{\Delta},
\]
ensuring that all intervals are non-overlapping within their respective groups. These intervals are visualized in Figure~\ref{fig:bad-instance}.

\begin{figure}[!hbt]
    \centering
    \begin{tikzpicture}[scale=1]
        % Define intervals as {start, end, y}
        \foreach \xstart/\xend/\y in {
            0.4/0/5
        } {
            \drawInterval[MyBlue]{solid}{\xstart}{\xend}{\y};
        }
        \foreach \xstart/\xend/\y in {
            0.8/3.7/4.7,
            1.2/3.7/4.7,
            2/3.7/4.7
        } {
            \drawInterval[gray]{solid}{\xstart}{\xend}{\y};
        }
        \foreach \xstart/\xend/\y in {
            0.0/0.2/1.2,
            0.0/1.4/2.4,
            0.0/3.6/4.6,
            0.0/4.8/9.8
        } {
            \drawInterval[MyRed]{solid}{\xstart}{\xend}{\y};
        }

        \node[gray] at (3, 0) {$\cdots$};
        \node[gray] at (4.2, 1.7) {$\vdots$};
        
    \end{tikzpicture}
    \caption{Illustration of prediction $\predI$. The red intervals represent the optimal solution for this set of intervals.}
    \label{fig:bad-instance}
\end{figure}
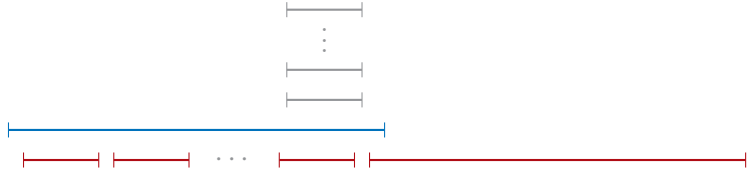

Note that the optimal solution for instance $\predI$ consists of 
\[
\bigcup_{j \in \{2,\dots,\lceil \Delta \rceil\} \cup \{n\}} I_j,
\]
and has a total value of 
\[
(\lceil\Delta\rceil - 1) \cdot \frac{k}{\Delta} + k.
\]

We are now ready to construct the adversarial instance $\Iadversary$.

\paragraph{Adversarial Instance $\Iadversary$.}
We start by releasing the interval $I_1$ as described above. We distinguish two cases based on whether $I_1$ is selected by $\alg'$ or not (see also Figure~\ref{fig:theorem1-adversary-instance}):

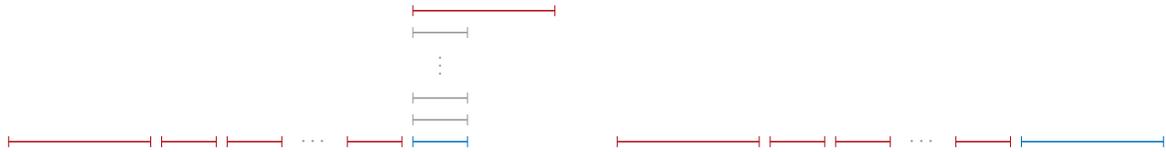
\begin{figure}[!hbt]
    \centering
    \begin{subfigure}{0.48\columnwidth}
    \resizebox{\columnwidth}{!}{
    \begin{tikzpicture}
        % Define intervals as {start, end, y}
        \foreach \xstart/\xend/\y in {
            0/0/2.6,
            0/2.8/3.8,
            0/4/5,
            0/6.2/7.2,
            2.4/7.4/10
        } {
            \drawInterval[MyRed]{solid}{\xstart}{\xend}{\y};
        }
        \foreach \xstart/\xend/\y in {
            0/7.4/8.4
        } {
            \drawInterval[MyBlue]{solid}{\xstart}{\xend}{\y};
        }
        \foreach \xstart/\xend/\y in {
            0.4/7.4/8.4,
            0.8/7.4/8.4,
            2/7.4/8.4
        } {
            \drawInterval[gray]{solid}{\xstart}{\xend}{\y};
        }

        \node[gray] at (5.6, 0) {$\cdots$};
        \node[gray] at (7.9, 1.5) {$\vdots$};
    
    \end{tikzpicture}
    }
    \caption{The adversarial input when the \alg\ accepts a short interval.}
    \label{fig:adversary-instance-case1}
    \end{subfigure}
    \hfill
    \begin{subfigure}{0.48\columnwidth}
        \resizebox{\columnwidth}{!}{
        \begin{tikzpicture}
            % Define intervals as {start, end, y}
            \foreach \xstart/\xend/\y in {
                0/0/2.6,
                0/2.8/3.8,
                0/4/5,
                0/6.2/7.2
            } {
                \drawInterval[MyRed]{solid}{\xstart}{\xend}{\y};
            }
            \foreach \xstart/\xend/\y in {
                0/7.4/10
            } {
                \drawInterval[MyBlue]{solid}{\xstart}{\xend}{\y};
            }
    
            \node[gray] at (5.6, 0) {$\cdots$};        
        \end{tikzpicture}
        }
        \caption{The adversarial input when the \alg\ never accepts a short interval.}
        \label{fig:adversary-instance-case2}
    \end{subfigure}
    \caption{The two cases of the adversarial input. In both figures, the red intervals represent the intervals in the optimal solution. In case~\ref{fig:adversary-instance-case1}, the blue interval appears only in the \alg, whereas in case~\ref{fig:adversary-instance-case2}, it is included in both the optimal solution and the \alg.}
    \label{fig:theorem1-adversary-instance}
\end{figure}

\begin{itemize}
    \item \textbf{Case 1: $\alg'$ accepts $I_1$} \\
    In this case, we set $\Iadversary = \predI$, and argue that $\alg'$ cannot satisfy the required consistency. Note that the prediction is perfectly accurate, and yet $\alg'$ obtains only a total value of $k$ on $\Iadversary$. The value of the optimal solution is
    \[
    (\lceil\Delta\rceil - 1) \cdot \frac{k}{\Delta} + k,
    \]
    resulting in the following ratio between the values of $\opt$ and $\alg'$:
    \[
    \frac{(\lceil\Delta\rceil - 1)\cdot \frac{k}{\Delta} + k}{k} = 1 + \frac{\lceil \Delta \rceil - 1}{\Delta}.
    \]    

    \item \textbf{Case 2: $\alg'$ rejects $I_1$} \\ 
    In this case, the remaining part of the instance consists of $\lceil \Delta \rceil^2 + 1$ intervals that arrive sequentially. The first $\lceil \Delta \rceil^2$ intervals are short, and the final one is a long interval. The exact release times are defined adaptively based on the decisions made by $\alg'$. More specifically, for each interval $I_i$, if $\alg'$ does not select $I_i$, then $I_{i+1}$ is released after the deadline of $I_i$. Otherwise, $I_{i+1}$ and all subsequent intervals are released so that they overlap with $I_i$. Figure~\ref{fig:general-adversary-instance} depicts the two different possibilities for $\Iadversary$.
    
    There are two subcases to consider:
    
    \begin{itemize}
        \item \textbf{$\alg'$ selects some short interval $I_i$.} \\
        By construction, this is the only interval that $\alg'$ can select, since all subsequent intervals overlap with $I_i$. On the other hand, the optimal algorithm can reject all short intervals and instead select the final long interval (and potentially some non-overlapping short ones). Thus, $|\alg'(\Iadversary)| \le \frac{k}{\Delta}$, while $|\opt(\Iadversary)| \ge k$.
    
        \item \textbf{$\alg'$ selects no short interval $I_i$.} \\
        In this case, $\alg'$ can only select the final long interval. However, the optimal algorithm can select the long interval as well as all $\lceil \Delta \rceil^2$ short intervals, which do not overlap with each other. That is, $|\alg'(\Iadversary)| \le k$, while
        \[
        |\opt(\Iadversary)| \ge \lceil \Delta \rceil^2 \cdot \frac{k}{\Delta}.
        \]
    \end{itemize}
    
    Overall, this gives a competitive ratio of at least $\Delta$. Since the instance is constructed independently of the prediction, the competitive ratio between $\alg'(\Iadversary)$ and $\opt(\Iadversary)$ is $\Delta$. Additionally, $\alg'$ loses the first long interval, incurring an additive loss of $k$.

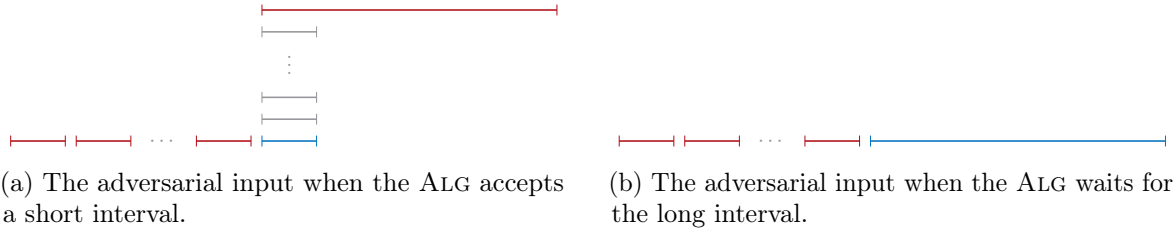
\begin{figure}[!hbt]
    \centering
    \begin{subfigure}{0.48\columnwidth}
    \resizebox{\columnwidth}{!}{
    \begin{tikzpicture}
        % Define intervals as {start, end, y}
        \foreach \xstart/\xend/\y in {
            0/0/1,
            0/1.2/2.2,
            0/3.4/4.4,
            2.4/4.6/10
        } {
            \drawInterval[MyRed]{solid}{\xstart}{\xend}{\y};
        }
        \foreach \xstart/\xend/\y in {
            0/4.6/5.6
        } {
            \drawInterval[MyBlue]{solid}{\xstart}{\xend}{\y};
        }
        \foreach \xstart/\xend/\y in {
            0.4/4.6/5.6,
            0.8/4.6/5.6,
            2/4.6/5.6
        } {
            \drawInterval[gray]{solid}{\xstart}{\xend}{\y};
        }

        \node[gray] at (2.8, 0) {$\cdots$};
        \node[gray] at (5.1, 1.5) {$\vdots$};
    
    \end{tikzpicture}
    }
    \caption{The adversarial input when the \alg\ accepts a short interval.}
    \label{fig:general-adversary-instance-case1}
    \end{subfigure}%
    \hfill
    \begin{subfigure}{0.48\columnwidth}
        \resizebox{\columnwidth}{!}{
        \begin{tikzpicture}
            % Define intervals as {start, end, y}
            \foreach \xstart/\xend/\y in {
                0/0/1,
                0/1.2/2.2,
                0/3.4/4.4
            } {
                \drawInterval[MyRed]{solid}{\xstart}{\xend}{\y};
            }
            \foreach \xstart/\xend/\y in {
                0/4.6/10
            } {
                \drawInterval[MyBlue]{solid}{\xstart}{\xend}{\y};
            }
    
            \node[gray] at (2.8, 0) {$\cdots$};        
        \end{tikzpicture}
        }
        \caption{The adversarial input when the \alg\ waits for the long interval.}
        \label{fig:general-adversary-instance-case2}
    \end{subfigure}
    \caption{The two cases of the adversarial input. In both figures, the red intervals represent the intervals selected by the optimal solution. In case~\ref{fig:general-adversary-instance-case1}, the blue interval is only selected by\alg, whereas in case~\ref{fig:general-adversary-instance-case2}, it is selected by both the optimal solution and \alg. The gray intervals are not selected by neither $\alg$ nor $\opt$. 
    }
    \label{fig:general-adversary-instance}
\end{figure}
    
\end{itemize}
    
\end{proof}

\section{\textsc{SemiTrust-and-Switch} Framework}
In this section, we turn to a more general variant of the \textsc{Trust-and-Switch} framework that relaxes its reliance on the predictions in the \firstPhase. The goal is to achieve a better consistency–robustness trade-off. The key idea is to define a length threshold and follow the prediction only for intervals whose lengths fall below this threshold. Although this approach sacrifices $1$-consistency unless the threshold is set very high, it allows for improved robustness guarantees when predictions are unreliable.

Intuitively, this method ensures that the algorithm does not miss promising long intervals in the case of inaccurate predictions, thereby improving its worst-case performance.

\paragraph{Framework \textsc{SemiTrust-and-Switch}.} 
In this framework, the algorithm uses a length threshold $\tau$ to determine whether to follow the prediction. More formally, in the \firstPhase, the algorithm accepts an interval if either (i) the prediction suggests it should be accepted, or (ii) the interval’s length exceeds the threshold $\tau$. As in \textsc{Trust-and-Switch}, the algorithm evaluates the predictions over time and switches to a classical algorithm in the \secondPhase\ using the same strategy.

Note that the framework uses the $\pred$ algorithm to evaluate the quality of the prediction. More specifically, it compares the value of the predicted solution with the offline optimum up to time $t_j$, and switches to the \secondPhase\ if
\[
|\opt(\I(r_j^\leftarrow, t_j^\leftarrow))| > |\pred(\I(r_j^\leftarrow, t_j^\leftarrow))|.
\]

Since \textsc{SemiTrust-and-Switch} uses the same evaluation procedure and switching strategy as the \textsc{Trust-and-Switch} framework, we have the following observation.

\begin{observation}\label{obs:SemiFramework-cons-noSwitch}
\textsc{SemiTrust-and-Switch} switches to the \secondPhase\ only if the predictions are inaccurate.
\end{observation}

We now analyze the performance of the \textsc{SemiTrust-and-Switch} framework as a function of the threshold parameter $\tau$. Recall that in the \firstPhase, the algorithm accepts an interval if either the prediction suggests it should be accepted, or its length exceeds $\tau$. If $\tau$ is set below the length of the shortest interval, the algorithm accepts all intervals greedily, behaving identically to \greedy. On the other hand, if $\tau > k$, where $k$ is the maximum interval length, then no interval is accepted based on length alone, and the algorithm behaves identically to \textsc{Trust-and-Switch}(\greedy), relying fully on predictions in the \firstPhase.

The interesting regime occurs when the threshold $\tau$ lies strictly between the shortest and longest interval lengths. Specifically, we assume throughout this section that
\[
\frac{k}{\Delta} < \tau < k,
\]
so that the algorithm accepts long intervals (with length $> \tau$) greedily, while relying on predictions only for short intervals (with length $\leq \tau$). This setting improves robustness by ensuring that long, high-value intervals are not missed due to inaccurate predictions, while still maintaining reasonable consistency guarantees.

\begin{lemma}\label{lem:SemiFramewok-cons}
The \textsc{SemiTrust-and-Switch} framework satisfies $(1 + \frac{k}{\tau})$-consistency.
\end{lemma}
\begin{proof}
Assume the predictions $\predO$ are accurate. We aim to show that the total length of the optimal solution is at most $(1 + \frac{k}{\tau})$ times the total length of the solution produced by \textsc{SemiTrust-and-Switch}. By Observation~\ref{obs:SemiFramework-cons-noSwitch}, the \textsc{SemiTrust-and-Switch} algorithm never switches to the \secondPhase\ and remains in the \firstPhase\ throughout the instance.

Let us compare the output of \textsc{SemiTrust-and-Switch} with the optimal offline solution. Note that any interval with length less than $\tau$ accepted by \textsc{SemiTrust-and-Switch} must have been predicted, and by the accuracy of the predictions, all such intervals also appear in the optimal solution.

We now remove these common short intervals (those with length $< \tau$) from both solutions. What remains are only long intervals (with length $\geq \tau$) in the output of \textsc{SemiTrust-and-Switch}. For each such long interval $I_i$ accepted by the algorithm. Consider all intervals in the optimal solution that conflict with $I_i$, and group them into a class corresponding to $I_i$.

We now analyze one such class. Since $I$ was accepted by the algorithm, all intervals in the optimal solution that conflict with $I$ must start after $r_i$ and end before $d_i$.

Therefore, the total length of intervals in the optimal solution within any class is at most:
\[
k + l_i,
\]
while the contribution of \textsc{SemiTrust-and-Switch} in this class is $l_i$. Thus, the competitive ratio in this class is at most
\[
\frac{k + l_i}{l_i} = 1 + \frac{k}{l_i} \leq 1 + \frac{k}{\tau},
\]
since $l_i \ge \tau$.
\end{proof}

\begin{figure}[t]
        \centering
        \begin{tikzpicture}[scale=1]
            \node[anchor=east] at (0, 2) {$\opt$};
            \node[anchor=east] at (0, 1) {$\alg$};

            \foreach \y/\xstart/\xend in {
                1/5/7.5,
                2/5.4/6.4,
                2/6.8/9.2
            } {
                \drawInterval[MyGreen, thick]{densely dashed}{\y}{\xstart}{\xend};
            }

            \foreach \y/\xstart/\xend in {
                1/1/4,
                2/2.4/3.4,
                2/3.6/5.1
            } {
                \drawInterval[MyBlue, thick]{solid}{\y}{\xstart}{\xend};
            }

            \foreach \y/\xstart/\xend in {
                2/0.2/2
            } {
                \drawInterval[MyRed, thick]{densely dotted}{\y}{\xstart}{\xend};
            }

            % Blue intervals (solid)
            \node[fill=gray, circle, inner sep=1pt, name=vi] at (6.25, 1) {};    % midpoint of I_i
            \node[fill=gray, circle, inner sep=1pt, name=v4] at (5.9, 2) {};     % midpoint of I_4
            \node[fill=gray, circle, inner sep=1pt, name=v5] at (8.0, 2) {};     % midpoint of I_5
            
            % Red intervals (densely dashed)
            \node[fill=gray, circle, inner sep=1pt, name=vp1] at (2.5, 1) {};    % midpoint of I'_1
            \node[fill=gray, circle, inner sep=1pt, name=v2] at (2.9, 2) {};     % midpoint of I_2
            \node[fill=gray, circle, inner sep=1pt, name=v3] at (4.35, 2) {};    % midpoint of I_3
            
            % Green interval (densely dotted)
            \node[fill=gray, circle, inner sep=1pt, name=v1] at (1.1, 2) {};     % midpoint of I_1

            \draw[ultra thin, gray] (vi) -- (v4);
            \draw[ultra thin, gray] (vi) -- (v5);
            \draw[ultra thin, gray] (vi) -- (v3);

            \draw[ultra thin, gray] (vp1) -- (v3);
            \draw[ultra thin, gray] (vp1) -- (v3);
            \draw[ultra thin, gray] (vp1) -- (v2);
            \draw[ultra thin, gray] (vp1) -- (v1);
            
            \node[] at (2.5, 0.7) {$I'_1$};
            \node[] at (6.25, 0.7) {$I_i$};        

            \node[] at (1.1, 2.3) {$I_1$};        
            \node[] at (2.9, 2.3) {$I_2$};
            \node[] at (4.35, 2.3) {$I_3$};
            \node[] at (5.9, 2.3) {$I_4$};
            \node[] at (8, 2.3) {$I_5$};

            \draw[densely dotted, gray] (5, 0) -- (5, 3);
            \draw[densely dotted, gray] (7.5, 0) -- (7.5, 3);
            \draw[densely dotted, gray] (10, 0) -- (10, 3);
            \node[gray] at (5, 0.5) {$r_i$};
            \node[gray] at (7.5, 0.5) {$d_i$};
            \node[gray] at (10, 0.5) {$d_i + k$};

        \end{tikzpicture}
        \caption{The last connected component of the conflict graph $G(\opt, \alg)$ before switching. The styles of intervals represent their mapped counterparts. Note that $I_1$ is the only interval taken by \opt\ that is not mapped to any interval. All intervals $I_4$ and $I_5$, which are mapped to $I_i$, are contained within the range $[r_i, d_i + k]$.}
        \label{fig:intervals-general-robust-proof}
    \end{figure}
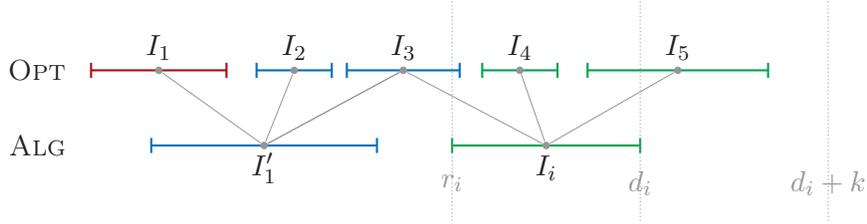

\begin{theorem}\label{Thm:SemiFramewok-robustness}
The \textsc{SemiTrust-and-Switch} framework $\alg$ satisfies $(1 + \frac{k}{\tau})$-consistency. Moreover, if the \secondPhase\ uses a $\theta$-competitive algorithm, then the framework satisfies
\[
|\opt| \le \max(\theta, \Delta + 1) \cdot |\alg| + \tau.
\]
\end{theorem}
\begin{proof}
    Let $A$ and $B$ be two feasible sets of intervals. Define conflict graph $G(A, B)$ as an undirected graph with $|A| + |B|$ nodes, where each node corresponds to an interval from either $A$ or $B$. If an interval appears in both sets, it is represented by two distinct nodes, one for each set. An edge is placed between two nodes if their corresponding intervals intersect.

    Now consider the graph $G(\opt, \alg)$. It may consist of several connected components. For all but the last component (i.e., before the point where the framework switches), we must have $|\alg| \geq |\opt|$; otherwise, the framework would have detected suboptimality and switched earlier.
    
    We now analyze the final component. We aim to show that within this component, we have:
    \[
    |\opt| \le (\Delta + 1) \cdot |\alg| + \tau.
    \]
    
    For each interval $I_i$ in $\alg$, we assign to $I_i$ all intervals in $\opt$ whose release times lie in the interval $\left(r_i, d_i\right]$. Since this forms a connected component, there is at most one interval in $\opt$ that is not assigned to any interval in $\alg$. Moreover, since the algorithm did not accept this interval, its length must be less than $\tau$. Figure~\ref{fig:intervals-general-robust-proof} presents a visual summary.
    
    Now consider an arbitrary interval $I_i \in \alg$ within this component. The total length of the intervals in $\opt$ assigned to $I_i$ cannot exceed $l_i + k$. This is because all such intervals must fit within the window $\left(r_i, d_i\right]$ and cannot overlap with $I_i$, which has length $l_i$. Therefore,
    \[
    l_i + k \le (\Delta + 1) \cdot l_i,
    \]
    where we use the fact that $l_i \ge \frac{k}{\Delta}$.
    
    Summing over all intervals in $\alg$ and adding the at most $\tau$ contribution from the one unassigned interval in $\opt$, we obtain:
    \[
    |\opt| \le (\Delta + 1) \cdot |\alg| + \tau.
    \]
    For the second part of the instance (after the switch), the algorithm switches and ran a $\theta$-competitive algorithm in the \secondPhase. .

    Combining both parts, we get:
    \[
    |\opt| \le \max(\theta, \Delta + 1) \cdot |\alg| + \tau.
    \]
    
    This completes the proof.

\end{proof}

\subsection{Lower Bound}

In this subsection, we show a lower bound on the robustness of algorithms with consistency better than $\Delta$. Even with access to full information through $\predI$, no algorithm with consistency strictly better than $\Delta$ can guarantee a robustness bound better than $|\opt| \le \Delta \cdot |\alg| + \frac{k}{\Delta}.$

\begin{lemma}\label{lem:Semi-lowerbound}
Given two-value instances, no algorithm with robustness better than $\Delta \cdot |\alg| + \frac{k}{\Delta}$ can achieve consistency better than $\Delta$, even with full predictions $\predI$.
\end{lemma}

\begin{proof}
We want to show that for any deterministic algorithm $\alg$ that guarantees robustness of 
\[
|\opt(\I)| \leq \Delta \cdot |\alg(\I)| + \frac{k}{\Delta} - \epsilon
\]
for all instances $\I$, there exists an adversarial instance $\I_{\text{Adv}}$ on which $\alg$ has a competitive ratio of at least $\Delta$.

\paragraph{Predicted Instance $\predI$.} 
Consider a predicted instance $\predI$ with $\Delta^2 + 2$ intervals. The first interval is short, followed by a long interval that overlaps it. Then, $\Delta^2$ short intervals arrive, all overlapping with each other and ending before the deadline of the initial long interval. Finally, a second long interval arrives after all others. This instance is illustrated in Figure~\ref{fig:bad-instance}.

Now, consider the scenario where $\predI$ is given to the algorithm $\alg$. We construct an adversarial instance $\Iadversary$ based on the algorithm's decision.

\paragraph{Adversarial Instance.} 
The first interval in $\Iadversary$ is a short interval, identical to the first interval in $\predI$. We distinguish two cases based on the algorithm's decision for this interval.

\begin{itemize}
    \item \textbf{Case 1: The algorithm $\alg$ accepts the first short interval.} \\
    In this case, we set $\Iadversary = \predI$. The optimal solution has total length $k$, i.e., $|\opt(\Iadversary)| = k$, while the algorithm selects only the short interval, so $|\alg(\Iadversary)| = \frac{k}{\Delta}$. Since the prediction is accurate, this case witnesses a consistency ratio of $\Delta$.

    \item \textbf{Case 2: The algorithm $\alg$ rejects the first short interval.} \\
    In this case, the remaining part of the instance consists of $\lceil \Delta \rceil^2 + 1$ intervals that arrive sequentially. The first $\lceil \Delta \rceil^2$ intervals are short, and the final one is a long interval. The exact release times are defined adaptively based on the decisions made by $\alg$. More specifically, for each interval $I_i$, if $\alg$ does not select $I_i$, then $I_{i+1}$ is released after the deadline of $I_i$. Otherwise, $I_{i+1}$ and all subsequent intervals are released so that they overlap with $I_i$.
    
    Figure~\ref{fig:general-adversary-instance} depicts the two possible evolutions of $\Iadversary$ under this construction. The algorithm $\alg$ cannot achieve a competitive ratio better than $\Delta$ on the remaining part of the instance, and additionally incurs a loss of $\frac{k}{\Delta}$ due to rejecting the first short interval.
    
    The behavior of both $\opt$ and $\alg$ is illustrated in Figure~\ref{fig:theorem2-adversary-instance}.

\end{itemize}

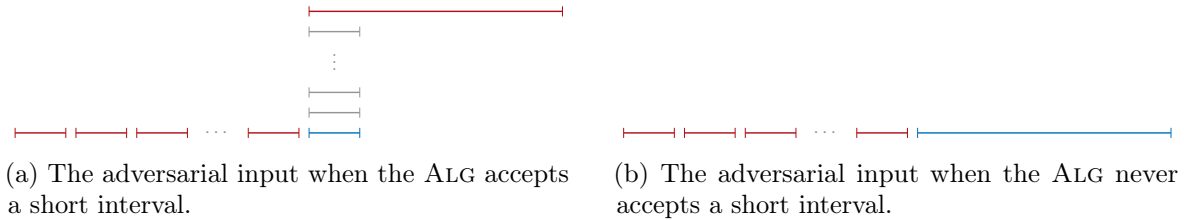
\begin{figure}[!hbt]
    \centering
    \begin{subfigure}{0.48\textwidth}
    \resizebox{\textwidth}{!}{
    \begin{tikzpicture}
        % Define intervals as {start, end, y}
        \foreach \xstart/\xend/\y in {
            0/4/5,
            0/5.2/6.2,
            0/6.4/7.4,
            0/8.6/9.6,
            2.4/9.8/14.8
        } {
            \drawInterval[MyRed]{solid}{\xstart}{\xend}{\y};
        }
        \foreach \xstart/\xend/\y in {
            0/9.8/10.8
        } {
            \drawInterval[MyBlue]{solid}{\xstart}{\xend}{\y};
        }
        \foreach \xstart/\xend/\y in {
            0.4/9.8/10.8,
            0.8/9.8/10.8,
            2/9.8/10.8
        } {
            \drawInterval[gray]{solid}{\xstart}{\xend}{\y};
        }

        \node[gray] at (8, 0) {$\cdots$};
        \node[gray] at (10.3, 1.5) {$\vdots$};
    
    \end{tikzpicture}
    }
    \caption{The adversarial input when the \alg\ accepts a short interval.}
    \label{fig:theorem2-adversary-instance-case1}
    \end{subfigure}%
    \hfill
    \begin{subfigure}{0.48\textwidth}
        \resizebox{\textwidth}{!}{
        \begin{tikzpicture}
            % Define intervals as {start, end, y}
            \foreach \xstart/\xend/\y in {
                0/4/5,
                0/5.2/6.2,
                0/6.4/7.4,
                0/8.6/9.6
            } {
                \drawInterval[MyRed]{solid}{\xstart}{\xend}{\y};
            }
            \foreach \xstart/\xend/\y in {
                0/9.8/14.8
            } {
                \drawInterval[MyBlue]{solid}{\xstart}{\xend}{\y};
            }
    
            \node[gray] at (8, 0) {$\cdots$};        
        \end{tikzpicture}
        }
        \caption{The adversarial input when the \alg\ never accepts a short interval.}
        \label{fig:theorem2-adversary-instance-case2}
    \end{subfigure}
    \caption{The two cases of the adversarial input. In both figures, the red intervals represent the intervals in the optimal solution. In case~\ref{fig:adversary-instance-case1}, the blue interval appears only in the \alg, whereas in case~\ref{fig:adversary-instance-case2}, it is included in both the optimal solution and the \alg.}
    \label{fig:theorem2-adversary-instance}
\end{figure}

\end{proof}

\section{Randomized Setting}
In this section, we study randomized algorithms for online interval scheduling with predictions. Our goal is to understand how randomness can improve performance guarantees. We first show that if we plug a randomized algorithm into the \secondPhase\ of the \textsc{Trust-and-Switch} framework, we obtain meaningful consistency–robustness trade-offs. Our main result is the \textsc{SmoothMerge} algorithm, whose performance degrades gracefully with the quality of the prediction. The key idea is to combine two algorithms, one that uses predictions and another that is prediction-agnostic, in a smooth and controlled manner.

\subsection{The \textsc{Trust-and-Switch} Framework}
The \textsc{Trust-and-Switch} framework naturally extends to this setting by allowing a randomized algorithm to be used in the \secondPhase.

\virtualalgorithm, introduced by~\cite{Lipton1994OnlineIS}, achieves a competitive ratio of $2$ on two-value instances. Briefly, this algorithm accepts any available interval of length $k$ whenever possible. For intervals of length $1$, the algorithm flips a fair coin to decide whether to accept the interval immediately or to \emph{virtually} accept it—that is, to defer accepting the short interval until the end of its duration.

Therefore, using Lemmas~\ref{lem:Framewok-1-cons} and~\ref{lem:TrustFramewok-2value-robustness}, we obtain the following proposition.

\begin{proposition}
Given a two-value instance and predictions $\predO$, \textsc{Trust-and-Switch}(\virtualalgorithm) satisfies $1$-consistency. Moreover, it guarantees
\[
|\opt| \le 2 \cdot |\textsc{Trust-and-Switch}(\virtualalgorithm)| + k,
\]
independent of the quality of the predictions.
\end{proposition}

For general instances,~\cite{Lipton1994OnlineIS} introduced a randomized \textsc{Marriage Algorithm} that achieves a competitive ratio of $O((\log \Delta)^{1+\epsilon})$, where %$\Delta$ is the ratio of the longest to the shortest interval, and 
$\epsilon > 0$ is an arbitrarily small constant.
Therefore, using Theorem~\ref{Thm:Framewok-robustness}, we obtain the following proposition.

\begin{proposition}
Given predictions $\predO$, \textsc{Trust-and-Switch}(\textsc{Marriage Algorithm}) satisfies $1$-consistency. Moreover, it guarantees
\[
|\opt| \le \max\left\{2,\ O\left((\log \Delta)^{1+\epsilon}\right)\right\} \cdot |\alg| + 2k,
\]
where $\alg$ denotes \textsc{Trust-and-Switch}(\textsc{Marriage Algorithm}), independent of the quality of the predictions.
\end{proposition}

\subsection{A Smooth Algorithm}

In this section, we introduce the \SmoothMerge\ algorithm, which achieves smoothness while maintaining robustness. The key insight is to combine two complementary algorithms: one that follows the prediction and performs well when the predictions are good (i.e., low error), and another classic algorithm that maintains reasonable performance regardless of prediction quality. Specifically, we design a randomized algorithm by merging the \trust\ and \greedy\ algorithms.

\paragraph{\SmoothMerge.} 
Upon the arrival of each interval, the algorithm proceeds as follows: if the interval conflicts with any previously accepted interval, it is immediately rejected. Otherwise, if both simulated strategies agree on accepting (or rejecting) the interval, the algorithm follows that unanimous decision. When the strategies disagree, the algorithm accepts the interval with probabilities $p_t$ and $p_g$, corresponding to the \trust\ and \greedy\ strategies, respectively. The full procedure is given in Algorithm~\ref{algorithm:smoothmerge}.

\begin{algorithm}[t]
\caption{\SmoothMerge}
\label{algorithm:smoothmerge}
\begin{algorithmic}[1]
\REQUIRE Interval set $\I$, Prediction $\predO$, probabilities $p_t, p_g$
\ENSURE Selected interval set $S$

\STATE $S \leftarrow \emptyset$

\FOR{each interval $I$ arriving in online order}
    \STATE $A_{\trust} \leftarrow$ \textbf{true} if \trust\ accepts $I$
    \STATE $A_{\greedy} \leftarrow$ \textbf{true} if \greedy\ accepts $I$ 

    \IF{$I$ conflicts with any interval in $S$}
        \STATE \textbf{reject} $I$ 
    \ELSIF{$A_{\trust} = \text{true}$ \textbf{and} $A_{\greedy} = \text{true}$}
        \STATE $S \leftarrow S \cup \{I\}$
    \ELSIF{$A_{\trust} = \text{true}$ \textbf{and} $A_{\greedy} = \text{false}$}
        \STATE $S \leftarrow S \cup \{I\}$ with \textbf{probability} $p_t$ 
    \ELSIF{$A_{\trust} = \text{false}$ \textbf{and} $A_{\greedy} = \text{true}$}
        \STATE $S \leftarrow S \cup \{I\}$ with \textbf{probability} $p_g$ 
    \ELSE
        \STATE \textbf{reject} $I$
    \ENDIF
\ENDFOR
\RETURN $S$
\end{algorithmic}
\end{algorithm}

To analyze the performance of our algorithm based on the quality of the predictions, we need to formally define prediction error.

\begin{definition}
\label{def:prediction-error}
For a given instance $\I$ and prediction $\predO$, the \emph{prediction error} $\eta$ is defined as:
\begin{align*}
\eta(\I, \predO) = \frac{|\opt(\I)| - |\trust(\I, \predO)|}{|\opt(\I)|}.
\end{align*}
where $|\opt(\I)|$ is the value of the optimal solution and $|\trust(\I, \predO)|$ is the value obtained by the \trust\ algorithm following prediction $\predO$.
\end{definition}

The prediction error $\eta \in [0, 1]$ quantifies the reliability of predictions: $\eta = 0$ indicates perfect predictions (where \trust\ achieves optimal performance), while larger values indicate less reliable predictions. When $\eta = 1$, the prediction-based algorithm performs arbitrarily poorly compared to the optimum.

\begin{theorem}
\label{thm:smooth-merging}
\SmoothMerge\ achieves both smoothness and robustness guarantees. Specifically,
\begin{align*}
\mathbb{E}[|\alg|] \geq 
\max\Big\{& |\opt| \cdot (1 - \eta) \cdot p_t \cdot (1 - p_g), \\
          & \frac{p_g - p_t p_g}{1 - p_t p_g} \cdot |\greedy| \Big\},
\end{align*}
where $p_t$ and $p_g$ are the chosen probabilities for the \trust\ and \greedy\ algorithms, respectively.
\end{theorem}

\begin{proof}
We start by analyzing the acceptance probabilities for intervals in both $\greedy(\I)$ and $\trust(\I)$.
For any interval $I_i \in \I$, we compute the probability that Algorithm~\ref{algorithm:smoothmerge} accepts $I_i$. The key insight is that intervals belonging to both $\greedy(\I)$ and $\trust(\I)$ are accepted with probability 1, as they cannot conflict with previously selected intervals.

Let conflict graph $G(\greedy(\I), \trust(\I))$ be an undirected graph with $|\greedy(\I)| + |\trust(\I)|$ nodes, where each node corresponds to an interval from either $\greedy(\I)$ or $\trust(\I)$. If an interval appears in both sets, it is represented by two distinct nodes, one for
each set. An edge is placed between two nodes if their corresponding intervals intersect.

We analyze intervals of each connected components of $G(\greedy(\I), \trust(\I))$ separately, as they are independent of each other. Consider a connected component with $n$ intervals from $\greedy(\I)$ (denoted $I_{G_0}, I_{G_1}, \ldots, I_{G_{n-1}}$) and $m$ intervals from $\trust(\I)$ (denoted $I_{T_0}, I_{T_1}, \ldots, I_{T_{m-1}}$), ordered by release time. Note that $I_{G_0}$ has an earlier release time than $I_{T_0}$ due to the \greedy\ algorithm's property of always selecting the earliest available interval.

For the earliest interval $I_{G_0}$ in a component, we can easily find the probability of accepting this interval by
\begin{equation}
\nonumber
\mathbb{P}[I_{G_0} \in \alg(\I)] = p_g.
\end{equation}
For any subsequent interval $I_{G_i}$ that conflicts with an earlier interval $I_{T_j}$, we can compute the probability recursively
\begin{equation}
\label{eq:pg-recursion}
\mathbb{P}[I_{G_i} \in \alg(\I)] = \bigl(1 - \mathbb{P}[I_{T_j} \in \alg(\I)]\bigr) \cdot p_g.
\end{equation}
Similarly, for any interval $I_{T_i} \in \trust(\I)$ that conflicts with an earlier interval $I_{G_j}$,
\begin{equation}
\label{eq:pt-recursion}
\mathbb{P}[I_{T_i} \in \alg(\I)] = \bigl(1 - \mathbb{P}[I_{G_j} \in \alg(\I)]\bigr) \cdot p_t.
\end{equation}
Solving the recurrence relations~\eqref{eq:pg-recursion} and~\eqref{eq:pt-recursion}, we obtain:
\begin{equation}
    \nonumber
    \mathbb{P}[I_{T_i} \in \alg(\I)] = (p_t - p_tp_g)\left(\frac{1 - (p_tp_g)^{x_i}}{1 - p_tp_g}\right),
\end{equation}
where $x_i$ is the depth in the recursion sequence, which is at most $n + m$. In another way, $x_i$ is the number of times we should substitute the right-hand side of equation recursively, until we reach to the base case of the recursion with is $\mathbb{P}[I_{G_0} \in \alg(\I)]$.

\begin{observation}
\label{obs:pt-increasing}
The probability $\mathbb{P}[I_{T_i} \in \alg(\I)]$ is non-decreasing with respect to the recursion depth.
\end{observation}

With the Observation~\ref{obs:pt-increasing}, we have $\min_i \mathbb{P}[I_{T_i} \in \alg(\I)] = \mathbb{P}[I_{T_0} \in \alg(\I)]$. Therefore,
% \begin{equation}
% \label{eq:pt-lowerbound}
\begin{equation}
\nonumber
\mathbb{P}[I_{T_i} \in \alg(\I)] \geq \mathbb{P}[I_{T_0} \in \alg(\I)] = p_t \cdot (1 - p_g).
\end{equation}

% \end{equation}

For intervals in $\greedy(\I)$, we have
$\min_i \mathbb{P}[I_{G_i} \in \alg(\I)] = \mathbb{P}[I_{G_{m-1}} \in \alg(\I)]$, based on equation~\eqref{eq:pg-recursion} and Observation~\ref{obs:pt-increasing}. Hence,
\begin{align}
\mathbb{P}[I_{G_i} \in \alg(\I)] &\geq \mathbb{P}[I_{G_{n-1}} \in \alg(\I)] \nonumber \\
&= (1 - \mathbb{P}[I_{T_j} \in \alg(\I)]) \cdot p_g \nonumber \\
&\geq \left(1 -\frac{p_t - p_tp_g}{1 - p_tp_g}\right) \cdot p_g \nonumber \\
&= \frac{p_g - p_tp_g}{1 - p_tp_g}. \nonumber
\end{align}

The expected performance is bounded by contributions from $\greedy(\I)$, which complete the proof of robustness property.
\begin{align}
\mathbb{E}[|\alg(\I)|] &= \sum_{I_i \in \I} \ell_i \cdot \mathbb{P}[I_i \in \alg(\I)] \nonumber \\ 
&\geq \sum_{I_i \in \greedy(\I)} \ell_i \cdot \mathbb{P}[I_i \in \alg(\I)] \nonumber \\
&\geq \sum_{I_i \in \greedy(\I)} \ell_i \cdot \frac{p_g-p_tp_g}{1 - p_tp_g} \nonumber \\
&= |\greedy(\I)| \cdot \frac{p_g-p_tp_g}{1 - p_tp_g}.
\label{eq:robustness-merg}
\end{align}

Similarly, considering contributions from $\trust(\I)$:
\begin{align}
\mathbb{E}[|\alg(\I)|] &= \sum_{I_i \in \I} \ell_i \cdot \mathbb{P}[I_i \in \alg(\I)] \nonumber \\ 
&\geq \sum_{I_i \in \trust(\I)} \ell_i \cdot \mathbb{P}[I_i \in \alg(\I)] \nonumber \\
&\geq \sum_{I_i \in \trust(\I)} \ell_i \cdot p_t \cdot (1 - p_g) \nonumber \\
&= |\trust(\I)| \cdot p_t \cdot (1 - p_g).
\label{eq:smoothness-intermediate}
\end{align}

From Definition~\ref{def:prediction-error}, we have $|\trust(\I)| = |\opt(\I)| \cdot (1 - \eta)$.
\begin{equation*}
\mathbb{E}[|\alg(\I)|] \geq |\opt(\I)| \cdot (1 - \eta) \cdot p_t \cdot (1 - p_g).
\end{equation*}

Therefore, combining equations~\eqref{eq:robustness-merg} and~\eqref{eq:smoothness-intermediate}, we have

\begin{align*}
\mathbb{E}[|\alg(\I)|] \geq 
\max\Big\{& |\opt(\I)| \cdot (1 - \eta) \cdot p_t \cdot (1 - p_g), \\
          & \frac{p_g - p_t p_g}{1 - p_t p_g} \cdot |\greedy(\I)| \Big\},
\end{align*}

This completes the proof.
\end{proof}

Figure~\ref{fig:smoothness-robustness-trade-off} illustrates the trade-off between smoothness and robustness across different values of $p_t$ and $p_g$. Specifically, it visualizes the maximum robustness achievable for different levels of smoothness.

\begin{figure}[t]
    \centering
    \includegraphics[width=0.5\textwidth]{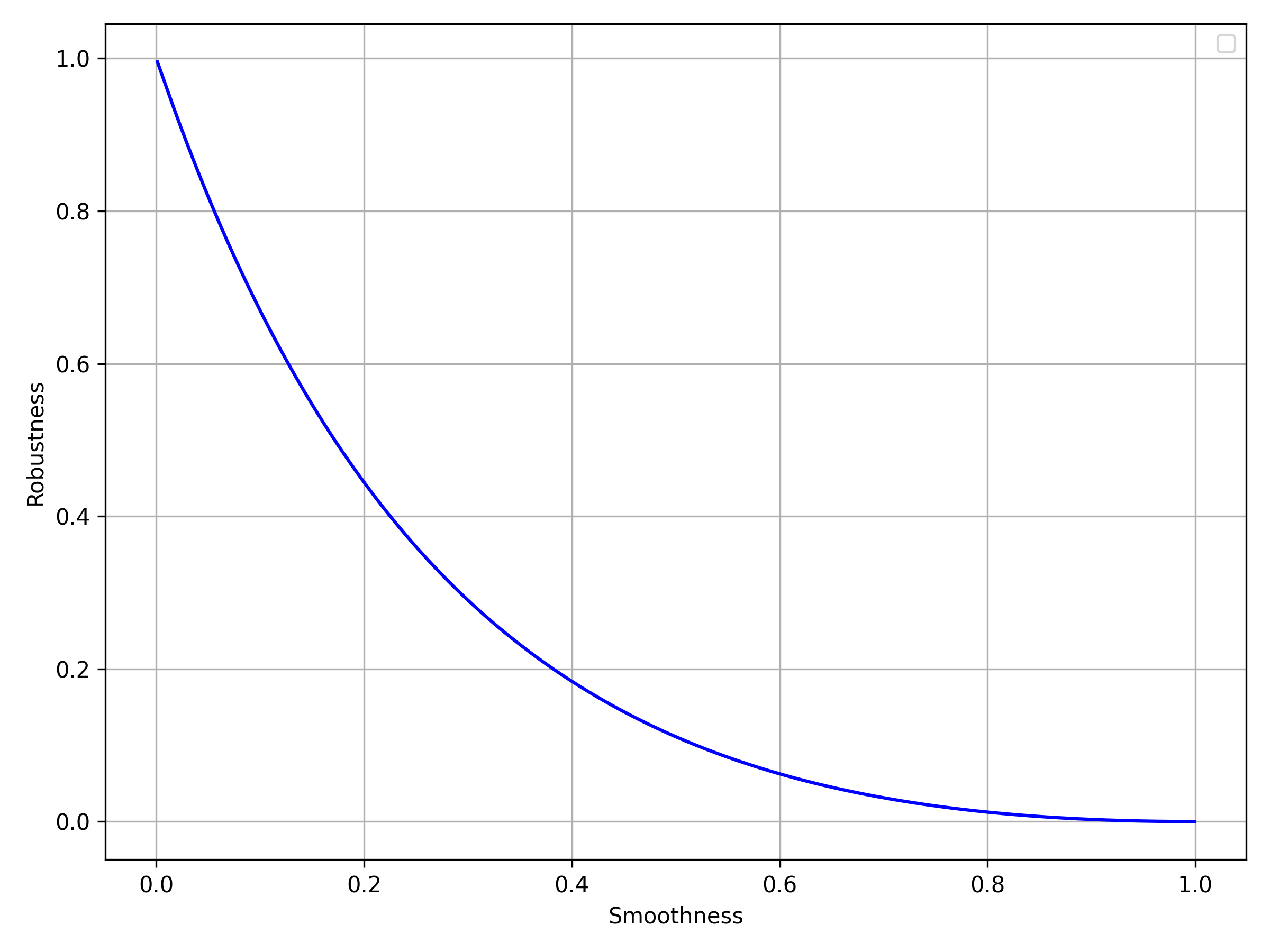}
    \caption{
        Trade-off curve between the best smoothness and robustness coefficients for \SmoothMerge\ as $p_t$ and $p_g$ vary. In the smoothness expression, the factor $|\opt(\I)| \cdot (1 - \eta)$ is omitted from the plot; in the robustness expression, the factor $|\greedy(\I)|$ is omitted.
    }
    \label{fig:smoothness-robustness-trade-off}
\end{figure}

When $p_t = 1$ and $p_g = 0$, the algorithm purely follows predictions, achieving perfect smoothness but no robustness. Conversely, when $p_t = 0$ and $p_g = 1$, the algorithm reduces to the standard \greedy\ approach, providing full robustness but no smoothness benefit. Other balanced choices of $p_t$ and $p_g$ offer both properties simultaneously, though with reduced coefficients. Table~\ref{tab:smoothmergevalues} illustrates the trade-offs between smoothness and robustness for different parameter choices.

\begin{table}[t]
    \centering
    \begin{tabular}{|c|c|c|c|}
        \hline
        $p_t$ & $p_g$ & Smoothness & Robustness \\
        \hline
        \hline
        1 & 0 & $(1 - \eta) \cdot |\opt(\I)|$ & $0$ \\
        \hline
        0 & 1 & $0$ & $|\greedy(\I)|$ \\
        \hline
        $0.50$ & $0.50$ & $0.25 \cdot (1 - \eta) \cdot |\opt(\I)|$ & $0.33 \cdot |\greedy(\I)|$ \\
        \hline
        $0.75$ & $0.33$ & $0.50 \cdot (1 - \eta) \cdot |\opt(\I)|$ & $0.11 \cdot |\greedy(\I)|$ \\
        \hline
        $0.50$ & $0.75$ & $0.12 \cdot (1 - \eta) \cdot |\opt(\I)|$ & $0.60 \cdot |\greedy(\I)|$ \\
        \hline
    \end{tabular}
    \caption{Smoothness and robustness coefficients for different values of $p_t$ and $p_g$.}
    \label{tab:smoothmergevalues}
\end{table}

This approach can be generalized to combine any pair of algorithms beyond \trust\ and \greedy. The key insight is to merge a prediction-dependent algorithm with a robust classical baseline. While it is natural to use a state-of-the-art classical algorithm in this role, our choice of the \greedy\ algorithm as the robust component is motivated by its simplicity and the tractability of its competitive analysis.

\subsubsection{Experimental Analysis}

We present an experimental evaluation of \SmoothMerge\ in comparison with \greedy\, \trust, and \opt. The algorithm smoothly interpolates between the behaviors of \trust\ and \greedy, adapting to prediction reliability through the choice of $p_t$ and $p_g$.

\paragraph{Experimental Setup.}
Our evaluation uses real-world scheduling data for parallel machines~\citep{10.1007/3-540-47954-6_4}. Table~\ref{tab:dataset} outlines the interval scheduling inputs we generated from the benchmarks of \cite{10.1007/3-540-47954-6_4}. For each benchmark with $n$ intervals, we consider $1000$ equally spaced values of $d \in [0, n]$. For each value of $d$, we randomly selected $\frac{n}{d}$ intervals and moved them all to the end of the timeline, making them conflict with each other. Then, we computed the offline optimal solution for this modified set of intervals and used this solution as the binary prediction $\predO$. To evaluate the \SmoothMerge\ algorithm, we calculated the average performance over $50$ runs on each prediction, using our personal laptop equipped with an Apple M3 GPU and $16$GB of memory.

Figure~\ref{fig:smoothmerge-performance} shows the result of \SmoothMerge\ algorithm for different values of $p_t$ and $p_g$. The results are aligned with our theoritical findings, \trust\ becomes worse than \SmoothMerge as the error value increases, while \SmoothMerge degrades gently as a function of the \greedy. We note that \greedy\ performs better when there is less overlap between the input intervals. In an extreme case, when no two intervals overlap, \greedy\ is trivially optimal.

\begin{table}[t]
    \centering
    \begin{tabular}{|c|c|c|c|}
        \hline
        Name & Input Size & $k$ & $\Delta$   
        \\
        \hline
        \hline
        NASA-iPSC-1993-3.1 & 18,065 & 62,643 & 62,643  \\
        \hline
        CTC-SP2-1996-3.1 & 77,205 & 71,998 & 71,998 \\
        \hline
    \end{tabular}
    \caption{Details on the benchmarks used in our experiment.}
    \label{tab:dataset}
\end{table}

\begin{figure}[t]
    \centering
    \begin{subfigure}[b]{0.48\textwidth}
        \centering
        \includegraphics[width=\textwidth]{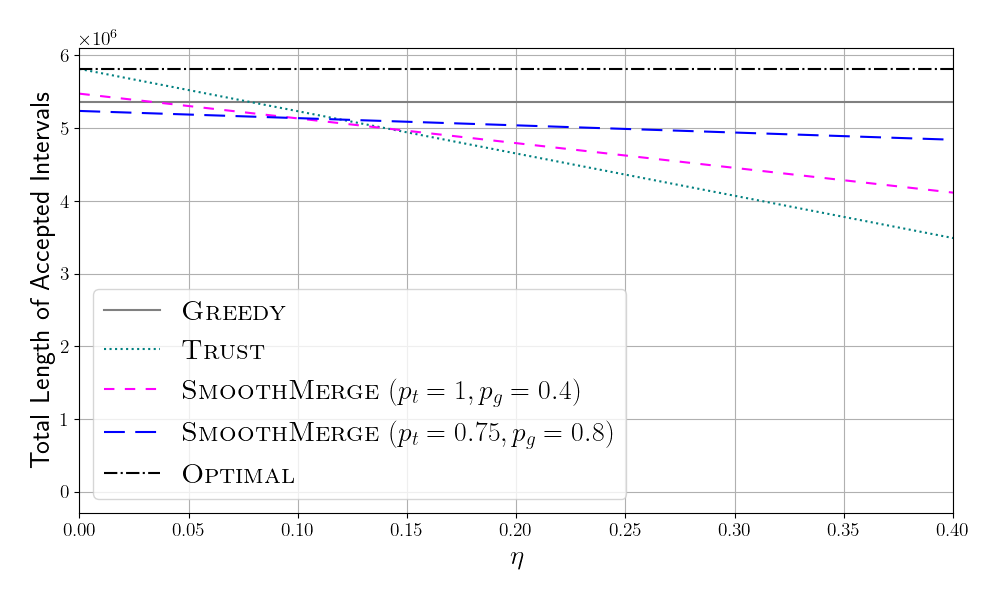}
        \caption{NASA-iPSC-1993-3.1 dataset}
        \label{fig:smoothmerge-performance-a}
    \end{subfigure}
    \hfill
    \begin{subfigure}[b]{0.48\textwidth}
        \centering
        \includegraphics[width=\textwidth]{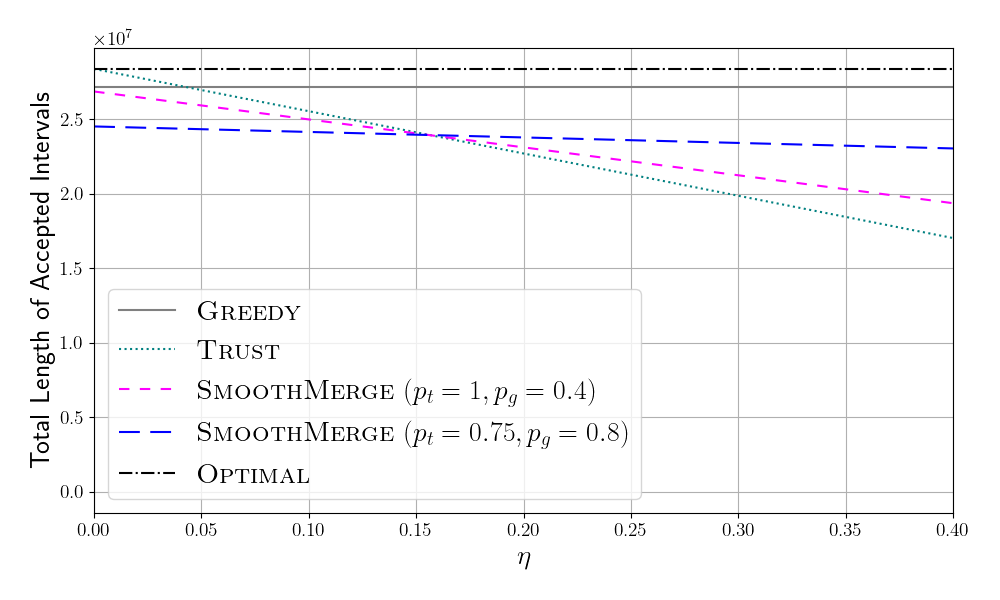}
        \caption{CTC-SP2-1996-3.1 dataset}
        \label{fig:smoothmerge-performance-b}
    \end{subfigure}
    \caption{
        Performance analysis of \SmoothMerge\ algorithm for two choices of parameters $p_t$ and $p_g$, compared to \opt, \trust, and \greedy.
    }
    \label{fig:smoothmerge-performance}
\end{figure}

\section{Conclusion}
We presented a systematic study of online interval scheduling with predictions in the irrevocable setting, focusing on maximizing the total length of accepted intervals. Our main contributions include the \textsc{SemiTrust-and-Switch} framework, which captures trade-offs between consistency and robustness. We showed that these frameworks are optimal for two-value instances in certain settings and can incorporate both deterministic and randomized algorithms. In addition, we introduced the \textsc{SmoothMerge} algorithm, which achieves smoothness by blending predictive and non-predictive strategies in a randomized manner.

Beyond the specific results presented here, our frameworks offer a modular approach that may be adaptable to other variants of online interval scheduling, such as revocable models, weighted objectives, or settings with resource constraints. While maximization problems present unique challenges for switching strategies, our results suggest that such approaches may be broadly useful in designing learning-augmented algorithms across a wider range of settings.

\pagebreak
\bibliographystyle{plainnat}
\bibliography{References}

\end{document}